\documentclass[twoside,11pt]{article}
\RequirePackage{color}
\RequirePackage{amsfonts}                   
\RequirePackage{amssymb}                   
\RequirePackage{amsmath}
\RequirePackage{mathrsfs}	
\usepackage{stmaryrd}		
\usepackage{natbib,eucal,graphics,bm,colortbl}
\usepackage{graphicx}
\usepackage[a4paper]{geometry}
\usepackage{hyperref}

\ifx\theorem\undefined
\newtheorem{theorem}{Theorem}
\newtheorem{lemma}[theorem]{Lemma}

\newtheorem{corollary}[theorem]{Corollary}
\newtheorem{definition}[theorem]{Definition}
\newtheorem{conjecture}[theorem]{Conjecture}
\fi

\ifx\BlackBox\undefined
\newcommand{\BlackBox}{\rule{1.5ex}{1.5ex}}  
\fi

\ifx\proof\undefined
\newenvironment{proof}{\par\noindent{\bf Proof\ }}{\hfill\BlackBox\\[2mm]}
\fi

\definecolor{tabgrey}{rgb}{0.8,0.8,0.8}

\newcommand{\EE}{\mathbb{E}}
\newcommand{\LL}{\mathbb{L}}
\newcommand{\PP}{\mathbb{P}}
\newcommand{\RR}{\mathbb{R}}

\newcommand{\Xsf}{\mathsf{X}}
\newcommand{\Ysf}{\mathsf{Y}}

\newcommand{\Xcal}{\mathcal{X}}
\newcommand{\Ycal}{\mathcal{Y}}
\newcommand{\Vcal}{\mathcal{V}}

\newcommand{\Hcal}{\mathcal{H}}
\newcommand{\Lcal}{\mathcal{L}}
\newcommand{\Scal}{\mathcal{S}}

\newcommand{\argmin}{\operatornamewithlimits{arg\,min}}

\newcommand{\inner}[2]{\left\langle #1,#2 \right\rangle}
\newcommand{\test}[1]{\ensuremath{\llbracket #1 \rrbracket}}
\newcommand{\loss}{\ell}
\newcommand{\reals}{\mathbb{R}}
\newcommand{\half}{\frac{1}{2}}
\newcommand{\thalf}{{\textstyle\half}}

\newcommand{\heta}{\hat{\eta}}
\newcommand{\minimal}[1]{\underline{#1}}
\newcommand{\minL}{\minimal{L}}
\newcommand{\minLL}{\minimal{\LL}}

\newcommand{\hh}{\hat{h}}
\newcommand{\Wb}{\overline{W}}
\newcommand{\lf}{\star}
\newcommand{\lfdual}[1]{{#1}^{\lf}}

\title{Composite Binary Losses} 
\author{Mark D. Reid\\
   Australian National University and NICTA\\ 
   Canberra ACT 0200, Australia\\
   {Mark.Reid@anu.edu.au}\\
\\
Robert C. Williamson\\
    Australian National University and NICTA\\ 
    Canberra ACT 0200, Australia\\
    {Bob.Williamson@anu.edu.au}
}

\begin{document}

\maketitle

\begin{abstract}
    We study losses for binary classification and class probability estimation
    and extend the understanding of them from margin losses to general
    composite losses which are the composition of a proper loss with a link
    function.  We characterise when margin losses can be proper composite
    losses, explicitly show how to determine a symmetric loss in full from half
    of one of its partial losses, introduce an intrinsic parametrisation of
    composite binary losses and give a complete characterisation of the
    relationship between proper losses and ``classification calibrated''
    losses. We also consider the question of the ``best'' surrogate binary
    loss. We introduce a precise notion of ``best'' and show there exist
    situations where two convex surrogate losses are incommensurable. We
    provide a complete explicit characterisation of the convexity of composite
    binary losses in terms of the link function and the weight function
    associated with the proper loss which make up the composite loss. This
    characterisation suggests new ways of ``surrogate tuning''. Finally, in an
    appendix we present some new algorithm-independent results on the
    relationship between properness,  convexity and robustness to
    misclassification noise for binary losses and show that all convex proper
    losses are non-robust to misclassification noise. 
\end{abstract}

\section{Introduction}

A \emph{loss} function is the means by which a learning algorithm's performance
is judged.  A \emph{binary} loss function is a loss for a supervised prediction
problem   where there are two possible labels associated with the examples.  A
\emph{composite} loss is the composition of a proper loss (defined below) and a
link function (also defined below). In this paper we study composite binary
losses and develop a number of new characterisation results.

Informally, proper losses are well-calibrated  losses for class probability
estimation, that is for the problem of not only predicting a binary
classification label, but providing an estimate of the probability that an
example will have a positive label. Link functions are often used to map the
outputs of a predictor to the interval $[0,1]$ so that they can be interpreted
as probabilities. Having such probabilities is often important in applications,
and there has been considerable interest in understanding how to get accurate
probability estimates \citep{Platt2000,GneitingRaftery2007,Cohen:2004} and 
understanding the implications of
requiring loss functions provide good probability 
estimates~\citep{Bartlett:2007}.

Much previous work in the machine learning literature has focussed on
\emph{margin losses} which intrinsically treat positive and negative classes
symmetrically.  However it is now well understood how important it is to
be able to deal with the non-symmetric case \citep{BachHeckermanHorvitz2006,
elkan2001fcs, BeygelzimerLangfordZadrozny2008,Buja:2005,ProvostFawcett2001}. A
key goal of the present work is to consider composite losses in the general
(non-symmetric) situation.

Having the flexibility to choose a loss function is important in order to
``tailor'' the solution to a machine learning problem; confer
\citep{Hand1994,HandVinciotti2003,Buja:2005}. Understanding the structure of
the set of loss functions and having natural parametrisations of them is useful
for this purpose. Even when one is using a loss as
a surrogate for the loss one would ideally like to minimise, it is helpful to
have an easy to use parametrisation --- see the discussion of ``surrogate
tuning'' in the Conclusion.

The paper is structured as follows.  In \S \ref{section:losses} we introduce
the notions of a loss, the conditional and full risk which we
will make extensive use of throughout the paper.

In \S  \ref{section:cpe-losses} we introduce losses for Class Probability
Estimation (CPE), define some technical properties of them, and present some
structural results.  We introduce and exploit Savage's characterisation of
proper losses and use it to characterise proper symmetric CPE-losses.

In \S \ref{sec:composite-losses} we define  composite losses formally and
characterise when a loss is a proper composite loss in terms of its partial
losses. We introduce a natural and intrinsic parametrisation of proper
composite losses and characterise when a margin loss can be a proper composite
loss. We also show the relationship between regret and Bregman divergences for
general composite losses.

In \S \ref{section:classification-calibrated} we characterise the relationship
between classification calibrated losses (as studied for example by
\citet{BartlettJordanMcAuliffe2006}) and proper composite losses.

In \S \ref{section:convexity}, motivated by the question of which is the
best surrogate
loss, we characterise when a proper
composite loss is convex in terms of the natural parametrisation of such
losses.

In \S \ref{section:choosing} we study surrogate losses making use of some of
the earlier material in the paper.  A \emph{surrogate} loss function is a loss
function which is not exactly what one wishes to minimise but is easier to work
with algorithmically.   We define a well founded notion of ``best'' surrogate
loss and show that some convex surrogate losses are incommensurable on some
problems.
We also study other notions of ``best'' and explicitly
determine the surrogate loss that has the best surrogate regret bound in a
certain sense.

Finally, in \S \ref{section:conclusion} we draw some more general conclusions.

Appendix \ref{section:appendix-robustness} builds upon some of the results in
the main paper and presents some new algorithm-independent results on the
relationship between properness,  convexity and robustness to misclassification
noise for binary losses and shows that all convex proper losses are non-robust
to misclassification noise.

\section{Losses and Risks}
\label{section:losses}
We write $x\wedge y:=\min(x,y)$ and $\test{p}=1$ if $p$ is true and
$\test{p}=0$ otherwise\footnote{This is the Iverson bracket notation as
recommended by \citet{Knuth1992}.}.  
The generalised function $\delta(\cdot)$ is defined by
$\int_a^b \delta(x) f(x)dx=f(0)$ when $f$ is continuous at $0$ and $a<0<b$.
Random variables are written in sans-serif font:  $\mathsf{X}$, $\mathsf{Y}$.

Given a set of 
labels $\Ycal := \{-1,1\}$ and a set of prediction values $\Vcal$ we will 
say a \emph{loss} is any function\footnote{Restricting the 
        output of a loss to $[0,\infty)$ is equivalent to 
	assuming the loss has a lower bound and then translating its output.}
$\ell : \Ycal \times \Vcal \to [0,\infty)$.
We interpret such a loss as giving a penalty $\ell(y,v)$ when predicting the 
value $v$ when an observed label is $y$. 
We can always write an arbitrary loss in terms of its \emph{partial losses}
$\ell_1:=\ell(1,\cdot)$ and $\ell_{-1}:=\ell(-1,\cdot)$ using
\begin{equation}
    \ell(y,v)= \test{y=1}\ell_1(v) + \test{y=-1}\ell_{-1}(v).
    \label{eq:decomposition-into-partial-losses}
\end{equation}

Our definition of a loss function covers all commonly used \emph{margin losses} 
(\emph{i.e.} those which can be expressed as $\ell(y,v) = \phi(yv)$ for some 
function $\phi:\RR\to[0,\infty)$) such as the 
\emph{0-1 loss} $\ell(y,v) = \test{yv > 0}$, the 
\emph{hinge loss} $\ell(y,v) = \max(1-yv,0)$, the 
\emph{logistic loss}  $\ell(y,v) = \log(1+e^{yv})$, and the 
\emph{exponential loss} $\ell(y,v) = e^{-yv}$ commonly used in boosting. 
It also covers \emph{class probability estimation losses} where the 
predicted values $\heta\in\Vcal=[0,1]$ are directly interpreted as probability
estimates.\footnote{
	These are known as \emph{scoring rules} in the statistical literature
	\citep{GneitingRaftery2007}.}
We will use $\heta$ instead of $v$ as an argument to indicate losses for class 
probability estimation and use the shorthand \emph{CPE losses} to distinguish 
them from general losses.
For example, 
\emph{square loss} has partial losses 
$\ell_{-1}(\heta) = \heta^2$ and 
$\ell_{1}(\heta) = (1-\heta)^2$, 
the \emph{log loss} 
$\ell_{-1}(\heta) = \log(1-\heta)$ and $\ell_1(\heta) = \log(\heta)$,
and the family of
\emph{cost-weighted misclassification losses} parametrised by $c\in(0,1)$ is
given by
\begin{equation}\label{eq:lc}
	\ell_c(-1,\heta) = c\test{\heta\ge c}
	\ \text{and}\ 
	\ell_c(1,\heta) = (1-c)\test{\heta<c}.
\end{equation}

\subsection{Conditional and Full Risks}

Suppose we have random examples $\Xsf$ with associated labels $\Ysf\in\{-1,1\}$
The joint distribution of
$(\Xsf,\Ysf)$ is denoted~$\PP$ and the marginal distribution of $\Xsf$ is
denoted $M$. Let the observation conditional density
$\eta(x):=\mathrm{Pr}(\Ysf=1 | \Xsf=x)$. 
Thus one can specify an experiment by either $\PP$ or $(\eta,M)$.

If $\eta\in[0,1]$ is the probability of observing the label $y=1$ the 
\emph{point-wise risk} (or \emph{conditional risk}) 
of the estimate $v\in\Vcal$ is defined as the 
$\eta$-average of the point-wise risk for $v$:
\[
L(\eta,v) := \EE_{\mathsf{Y}\sim\eta}[\loss(\mathsf{Y},v)] 
= \eta\ell_1(v)+(1-\eta)\ell_{-1}(v).
\]
Here, $\Ysf \sim \eta$ is a shorthand for labels being drawn from a Bernoulli 
distribution with parameter $\eta$.
When $\eta : \mathcal{X} \to [0,1]$ is an observation-conditional density, 
taking the $M$-average of the point-wise risk gives the \emph{(full) 
risk} of the estimator $v$, now interpreted as a function $v:\Xcal\to\Vcal$:
\[
	\LL(\eta,v,M) := 
	\EE_{\Xsf\sim M} [L(\eta(\Xsf),v(\Xsf))] .
\]
We sometimes write $\LL(v,\PP)$ for $\LL(\eta,v,M)$ where ($\eta,M)$
corresponds to the joint distribution $\PP$.
We write $\loss$, $L$ and $\LL$ for the loss, point-wise 
and full risk throughout this paper.
The \emph{Bayes risk} is the minimal achievable value of the risk and is
denoted
\[
	\minLL(\eta, M)
	:= \inf_{v \in \Vcal^{\mathcal{X}}} \LL(\eta,v,M)
	= \EE_{\mathsf{X}\sim M}\left[\minL(\eta(\mathsf{X}))\right],
\]
where 
\[
[0,1]\ni\eta\mapsto\minL(\eta) := \inf_{v\in\Vcal} L(\eta,v)
\]
is the \emph{point-wise} or \emph{conditional Bayes risk}. 

There has been increasing awareness of the importance of the conditional Bayes 
risk curve $\minL(\eta)$ --- also known as ``generalized entropy'' 
\citep{GrunwaldDawid2004} --- in the analysis of losses for probability estimation
\citep{Kalnishkan:2004,Kalnishkan:2007,Abernethy:2009,Masnadi-Shirazi:2009}.
Below we will see how it is effectively the curvature of $\minL$ that determines
much of the structure of these losses.

\section{Losses for Class Probability Estimation}
\label{section:cpe-losses}

We begin by considering CPE losses, that is, functions
$\ell:\{-1,1\}\times[0,1]\to[0,\infty)$ and briefly summarise a number of 
important existing structural results for \emph{proper losses} --- a large, 
natural class of losses for class probability estimation.

\subsection{Proper, Fair, Definite and Regular Losses}
There are a few properties of losses for probability estimation that we will 
require.
If $\heta$ is to be interpreted as an estimate of the true positive class
probability $\eta$ (\emph{i.e.}, when $y=1$) then it is desirable to require 
that $L(\eta,\heta)$ be minimised by $\heta = \eta$ for all $\eta\in[0,1]$. 
Losses that satisfy this constraint are said to be \emph{Fisher consistent} and 
are known as \emph{proper losses} \citep{Buja:2005,GneitingRaftery2007}. 
That is, a proper loss $\loss$ satisfies $\minL(\eta) = L(\eta,\eta)$ 
for all $\eta\in[0,1]$.  
A \emph{strictly proper} loss is a proper loss for which the minimiser of 
$L(\eta,\heta)$ over $\heta$ is unique. 

We will say a loss is \emph{fair} whenever
\begin{equation}
    \ell_{-1}(0) = \ell_1(1) = 0.
    \label{eq:loss-fair}
\end{equation}
That is, there is no loss incurred for perfect prediction. 
The main place fairness is relied upon is in the integral representation of 
Theorem~\ref{thm:schervish} where it is used to get rid of some constants of 
integration.
In order to explicitly construct losses from their associated ``weight functions'' 
as shown in Theorem~\ref{theorem:loss-from-weight}, we will require that the loss 
be \emph{definite}, that is, its point-wise Bayes risk for deterministic events 
(\emph{i.e.}, $\eta =0$ or $\eta=1$) must be bounded from below:
\begin{equation}
	\minL(0) > -\infty\ , \ \minL(1) > -\infty.
	\label{eq:loss-definite}
\end{equation}
Since properness of a loss ensures $\minL(\eta) = L(\eta,\eta)$ we see that a
fair proper loss is necessarily definite since $L(0,0)=\ell_{-1}(0)=0>-\infty$, 
and similarly for $L(1,1)$.
Conversely, if a proper loss is definite then the finite values $\ell_{-1}(0)$ 
and $\ell_1(1)$ can be subtracted from $\ell(0,\cdot)$ and $\ell(1,\cdot)$ to 
make it fair.

Finally, for Theorem~\ref{thm:savage} to hold at the endpoints of the unit
interval, we require a loss to be \emph{regular}\footnote{This 
is equivalent to the conditions of \citet{Savage:1971} and 
\citet{Schervish1989}.};
that is,
\begin{equation}
    \label{eq:loss-regular}
      \lim_{\eta\searrow 0} \eta \loss_{1}(\eta)=
      \lim_{\eta\nearrow 1} (1-\eta)\loss_{-1}(\eta)=0.
\end{equation}
Intuitively, this condition ensures that making mistakes on events that never 
happen should not incur a penalty.
In most of the situations we consider in the remainder of this paper will 
involve losses which are proper, fair, definite and regular.

\subsection{The Structure of Proper Losses}
A key result in the study of proper losses is originally due to 
\citet{ShufordAlbertMassengill1966} though our presentation follows that of
\citet{Buja:2005}.
It characterises proper losses for probability estimation via a constraint on 
the relationship between its partial losses.
\begin{theorem}\label{thm:partials}
	Suppose $\ell:\{-1,1\}\times[0,1]\to\RR$ is a loss and that its partial
	losses $\ell_1$ and $\ell_{-1}$ are both differentiable.
	Then $\ell$ is a proper loss if and only if
	for all $\heta\in(0,1)$
	\begin{equation}\label{eq:Shuford}
		\frac{-\ell'_1(\heta)}{1-\heta} 
		= \frac{\ell'_{-1}(\heta)}{\heta}
		= w(\heta)
	\end{equation}
	for some \emph{weight function}
	$w:(0,1)\to\reals^+$ such that 
	$\int_\epsilon^{1-\epsilon}w(c)\,dc<\infty$
	for all $\epsilon>0$.
\end{theorem}
The equalities in (\ref{eq:Shuford}) should be interpreted in the $L_1$ sense.

This simple characterisation of the structure of proper losses has a number of 
interesting implications. 
Observe from (\ref{eq:Shuford}) that if $\ell$ is proper, given $\ell_1$
we can determine $\ell_{-1}$ or vice versa.
Also, the partial derivative of the conditional risk can be seen to be the 
product of a linear term and the weight function:
\begin{corollary}\label{cor:proper-derivative}
	If $\ell$ is a differentiable proper loss then
	for all $\eta\in[0,1]$
	\begin{equation}\label{eq:proper-deriviate}
		\frac{\partial}{\partial\heta}L(\eta,\heta)
		= (1-\eta)\ell'_{-1}(\heta) + \eta\ell'_1(\heta)
		= (\heta - \eta)w(\heta).
	\end{equation}
\end{corollary}
Another corollary, observed by \citet{Buja:2005}, is
that the weight function is related to the curvature of the conditional Bayes 
risk $\minL$.
\begin{corollary}
	Let $\ell$ be a a twice differentiable\footnote{
		The restriction to differentiable losses can be removed in most cases
		if generalised weight functions---that is, possibly infinite but  
		defining a measure on $(0,1)$---are permitted. For example, the
		weight function for the 0-1 loss is $w(c) = \delta(c-\half)$.
	} proper loss with weight function $w$ defined as in equation 
	(\ref{eq:Shuford}).
	Then for all $c\in(0,1)$ its conditional Bayes risk $\minL$ satisfies
	\begin{equation}\label{eq:w-in-terms-of-L}
		w(c) = -\minL''(c).
	\end{equation}
\end{corollary}
One immediate consequence of this corollary is that the conditional Bayes risk 
for a proper loss is always concave.
Along with an extra constraint, this gives another characterisation of proper 
losses~\citep{Savage:1971,Reid:2009a}.
\begin{theorem} {\bf (Savage)}
	A loss function $\loss$ is proper if and only if its point-wise 
	Bayes risk $\minL(\eta)$ is concave and for each $\eta,\heta\in(0,1)$
	\begin{equation}
		L(\eta,\heta) = \minL(\heta) + (\eta - \heta)\minL'(\heta).
		\label{eq:savage-rep}
	\end{equation}
	Furthermore if $\ell$ is regular this characterisation also holds at 
	the endpoints $\eta,\heta\in\{0,1\}$.
	\label{thm:savage}
\end{theorem}
This link between loss and concave functions makes it easy to establish a
connection, as \citet{Buja:2005} do, between \emph{regret} 
\(
    \Delta L(\eta,\heta) := L(\eta,\heta)-\minL(\eta)
\) 
for proper losses and \emph{Bregman divergences}. 
The latter are generalisations of distances and are defined in terms of convex
functions.
Specifically, if $f:\Scal\to\RR$ is a convex function over some convex set  
$\Scal\subseteq\RR^n$ then its associated Bregman divergence\footnote{
    A concise summary of Bregman divergences and their properties 
    is given by \citet[Appendix A]{Banerjee:2005a}.
} is
\[
    D_f(s,s_0) := f(s) - f(s_0) - \inner{s-s_0}{\nabla f(s_0)}
\]
for any $s,s_0\in\Scal$, where $\nabla f(s_0)$ is the gradient of $f$ at $s_0$.
By noting that over $\Scal=[0,1]$ we have $\nabla f = f'$, these definitions 
lead immediately to the following corollary of Theorem~\ref{thm:savage}.
\begin{corollary}\label{cor:regret-bregman}
    If $\ell$ is a proper loss then its regret is the Bregman divergence 
    associated with $f=-\minL$. That is,
    \begin{equation}\label{eq:regret-bregman}
        \Delta L(\eta,\heta) = D_{-\minL}(\eta,\heta).
    \end{equation}
\end{corollary}

Many of the above results can be observed graphically by plotting the 
conditional risk for a proper loss as in Figure~\ref{fig:partials}.
Here we see the two partial losses on the left and right sides of the figure
are related, for each fixed $\heta$, by the linear map 
$\eta\mapsto L(\eta,\heta) = (1-\eta)\ell_{-1}(\heta) + \eta\ell_1(\heta)$.
For each fixed $\eta$ the properness of $\ell$ requires that these convex 
combinations of the partial losses (each slice parallel to the left and right 
faces) are minimised when $\heta=\eta$. 
Thus, the lines joining the partial losses are tangent to the conditional Bayes 
risk curve $\eta\mapsto\minL(\eta)=L(\eta,\eta)$ shown above the dotted 
diagonal.
Since the conditional Bayes risk curve is the lower envelope of these tangents
it is necessarily concave.
The coupling of the partial losses via the tangents to the conditional Bayes 
risk curve demonstrates why much of the structure of proper losses is determined
by the curvature of $\minL$ --- that is, by the weight function $w$.

\begin{figure}
    \begin{center}
	\includegraphics[width=0.8\textwidth]{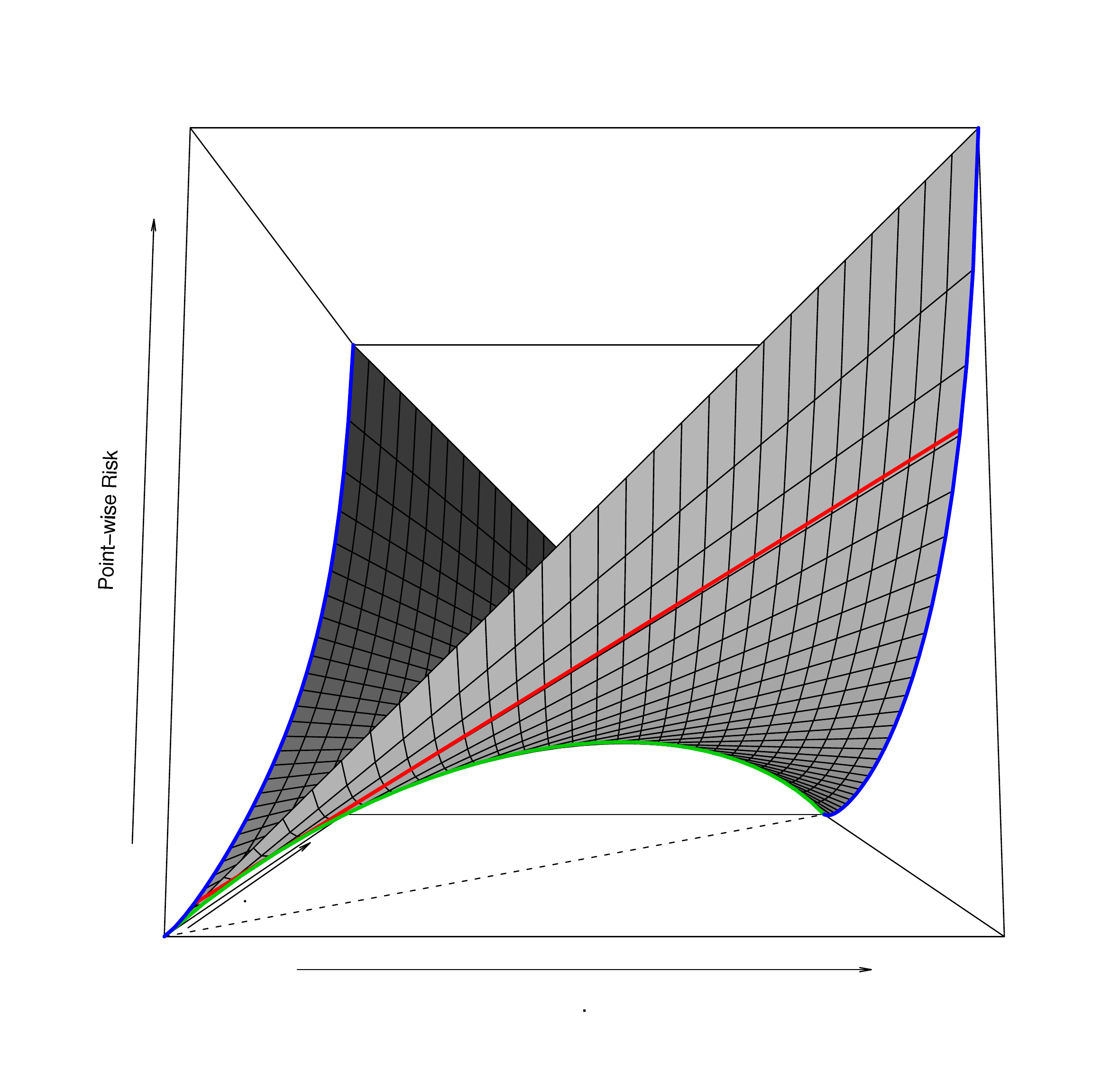}
	\end{center}

	\caption{\label{fig:partials}
		The structure of the conditional risk $L(\eta,\heta)$
		for a proper loss (grey surface). 
		The loss is log loss and its partials
		$\ell_{-1}(\heta) = -\log(\heta)$ and 
		$\ell_1(\heta) = -\log(1-\heta)$ shown on the left and right faces of 
		the box (blue curves).
		The conditional Bayes risk is the (green) curve above the dotted line
		$\heta=\eta$. 
		The (red) line connecting points on the partial loss curves shows the
		conditional risk for a fixed prediction $\heta$.
	}
\end{figure}
%
%
%

The relationship between a proper loss and its associated weight function is
captured succinctly via the following representation of proper losses as a
weighted integral of the cost-weighted misclassification losses $\ell_c$ defined
in (\ref{eq:lc}).
The reader is referred to \citep{ReidWilliamson2009} for the details, proof and 
the history of this result.
\begin{theorem}\label{thm:schervish}
Let $\loss : \mathcal{Y}\times[0,1]\to\RR$ be a fair, proper loss 
Then for each $\heta\in(0,1)$ and $y\in\mathcal{Y}$
\begin{equation}
	\loss(y,\heta) = \int_0^1 \loss_c(y,\heta)\,w(c)\,dc,
	\label{eq:schervish-representation}
    \end{equation}
where $w=-\minL''$.
Conversely, if $\ell$ is defined by (\ref{eq:schervish-representation}) for
some weight function $w:(0,1)\to[0,\infty)$ then it is proper.
\end{theorem}
Some example losses and their associated weight functions are given in
Table~\ref{table:losses}.
\citet{Buja:2005} show that  $\ell$ is strictly proper if and 
only if $w(c)>0$ in the sense that $w$ has non-zero mass on every open subset
of $(0,1)$. 
\begin{table}
    \begin{center}
    \arrayrulecolor{tabgrey}
    \setlength{\minrowclearance}{2mm}
    \renewcommand{\arraystretch}{1.3}
    \begin{tabular*}{\textwidth}{llll}
	\hline
	$w(c)$\hspace*{2cm}
	& $\ell_{-1}(\heta)$\hspace*{3.5cm}
	& $\ell_1(\heta)$\hspace*{3.5cm}
	& Loss
	\\[2mm]
	\hline\hline
	 $2\delta\left(\frac{1}{2} - c\right)$
	& $\test{\heta>\thalf}$
	& $\test{\heta\le\thalf}$
	& 0-1
	\\[2mm]
	\hline
	$\delta(c-c_0)$
	& $c_0\test{\heta\ge c_0}$
	& $(1-c_0)\test{\heta <c_0}$
	& $\ell_{c_0}$, $c_0\in [0,1]$
	\\[2mm]
	\hline
	 $\frac{1}{(1-c)^2 c}$
	& $\left[2\ln(1-\heta)+\frac{\heta}{1-\heta}\right]$
	& $\left[\ln \frac{1-\heta}{\heta} -1\right]$
	& ---
	\\[2mm]
	\hline
	 1
	& $\heta^2/2$
	& $(1-\heta)^2/2$
	& Square
	\\[2mm]
	\hline
	 $\frac{1}{(1-c)c}$
	& $-\ln(1-\heta)$
	& $-\ln(\heta)$
	& Log
	\\[2mm]
	\hline
	 $\frac{1}{(1-c)^2 c^2}$
	& $\left[\ln((1-\heta)\heta) - \frac{1-2\heta}{\heta}\right]$
	& $\left[\ln((1-\heta)\heta) + \frac{1-2\heta}{\heta}\right]$
	& ---
	\\[2mm]
	\hline
	 $\frac{1}{[(1-c)c]^{3/2}}$
	& $2\sqrt{\frac{\heta}{1-\heta}}$
	& $2\sqrt{\frac{1-\heta}{\heta}}$
	& Boosting
	\\[2mm]
	\hline
    \end{tabular*}
    \caption{Weight functions and associated partial
    losses.\label{table:losses}}
\end{center}
\end{table}
The following theorem from \citet{Reid:2009a} shows how to explicitly construct
a loss in terms of a weight function.
\begin{theorem}\label{theorem:loss-from-weight}
	Given a weight function $w : [0,1]\to[0,\infty)$, let 
	$W(t) = \int^t w(c)\,dc$ and $\Wb(t) = \int^t W(c)\,dc$. 
	Then the loss $\ell_w$ defined by
	\begin{equation}\label{eq:loss-from-w}
		\ell_w(y,\heta) = -\Wb(\heta) - (y-\heta)W(\heta)
	\end{equation}
	is a proper loss.
	Additionally, if $\Wb(0)$ and $\Wb(1)$ are both finite then
	\begin{equation}\label{eq:fair-loss-from-w}
		\ell_w(y,\heta) + (\Wb(1) - \Wb(0))y + \Wb(0)
	\end{equation}
	is a fair, proper loss.
\end{theorem}
Observe that if $w$ and $v$ are functions which differ on a set of measure zero
then they will lead to the same loss.
A simple corollary to Theorem~\ref{thm:schervish} is that
the partial losses are given by
\begin{equation}
    \ell_1(\heta)=\int_{\heta}^1 (1-c)w(c)dc\ \ \mbox{and}\ \ 
    \ell_{-1}(\heta)=\int_0^{\heta} c w(c)dc .
    \label{eq:partial-losses}
\end{equation}.
\subsection{Symmetric Losses}
We will say a loss is \emph{symmetric} if $\ell_1(\heta)=\ell_{-1}(1-\heta)$
for all $\heta\in [0,1]$. We say a weight function for a proper loss
or the conditional Bayes risk is
\emph{symmetric} if $w(c)=w(1-c)$ or $\minL(c)=\minL(1-c)$ for all $c\in[0,1]$.
Perhaps unsurprisingly, an immediate consequence of Theorem~\ref{thm:partials} 
is that these two notions are identical.

\begin{corollary}
A proper loss is symmetric if and only if its weight function is symmetric.
\end{corollary}

Requiring a loss to be proper and symmetric constrains the partial losses
significantly. 
Properness alone completely specifies one partial loss from the other.
Now suppose in addition that $\ell$ is symmetric. 
Combining $\ell_1(\heta)=\ell_{-1}(1-\heta)$ with 
(\ref{eq:Shuford}) implies
\begin{equation}
    \ell'_{-1}(1-\heta)= \frac{1-\heta}{\heta} \ell'_{-1}(\heta).
    \label{eq:symmetric-condition-derivative-partial-loss}
\end{equation}
This shows that $\ell_{-1}$ is completely determined by $\ell_{-1}(\heta)$ for
$\heta\in[0,\frac{1}{2}]$ (or $\heta\in [\frac{1}{2},1]$). Thus in order to
specific a symmetric proper loss, one needs to only specify one of the partial
losses on one half of the interval $[0,1]$. Assuming $\ell_{-1}$ is continuous 
at $\frac{1}{2}$ (or equivalently that $w$ has no atoms at $\frac{1}{2}$), 
by integrating both sides of
(\ref{eq:symmetric-condition-derivative-partial-loss}) we can derive an
explicit formula for the other half of $\ell_{-1}$ in terms of that which is
specified:
\begin{equation}
    \ell_{-1}(\heta) = \ell_{-1}({\textstyle\frac{1}{2}}) + 
    \int_{\frac{1}{2}}^{\heta} \frac{x}{1-x} \ell'_{-1}(1-x) dx,
    \label{eq:integral-formula-for-half-partial-loss}
\end{equation}
which works for determining $\ell_{-1}$ on either $[0,\frac{1}{2}]$ or
$[\frac{1}{2},1]$ when $\ell_{-1}$ is specified on $[\frac{1}{2},1]$ or
$[0,\frac{1}{2}]$ respectively (recalling the usual convention that
$\int_a^b=-\int_b^a$). We have thus shown:
\begin{theorem}
    If a loss is proper and symmetric, then it is completely determined
    by specifying one of the partial losses on half the unit interval (either
    $[0,\frac{1}{2}]$ or $[\frac{1}{2},0]$) and using
    (\ref{eq:symmetric-condition-derivative-partial-loss}) and
    (\ref{eq:integral-formula-for-half-partial-loss}).
\end{theorem}

We demonstrate (\ref{eq:integral-formula-for-half-partial-loss}) with four
examples. Suppose that $\ell_{-1}(\heta)=\frac{1}{1-\heta}$ for
$\heta\in[0,\frac{1}{2}]$. Then one can readily determine the complete partial
loss to be
\begin{equation}
    \ell_{-1}(\heta) = \frac{\test{\heta\le\frac{1}{2}}}{1-\heta} +
    \test{\heta>{\textstyle\frac{1}{2}}}\left(2+\log\frac{\heta}{1-\heta}\right).
\end{equation}
Suppose instead that
$\ell_{-1}(\heta)=\frac{1}{1-\heta}$ for
$\heta\in[\frac{1}{2},1]$. In that case we obtain
\begin{equation}
    \ell_{-1}(\heta) = \test{\heta\le\frac{1}{2}} 
    \left(2+\log\frac{\heta}{1-\heta}\right) +
    \frac{\test{\heta\ge\frac{1}{2}}}{1-\heta}.
\end{equation}
Suppose $\ell_{-1}(\heta)=\frac{1}{(1-\heta)^2}$ for $\heta\in[0,\frac{1}{2}]$.
Then one can determine that
\[
\ell_{-1}(\heta) = \frac{\test{\heta<\frac{1}{2}}}{(1-\heta)^2} +
\frac{\test{\heta\ge{\textstyle\frac{1}{2}}}\left(4+2(2\heta+\heta\log\heta-\heta\log(1-\heta)-1)\right)}{\heta}.
\]
Finally consider specifying that $\ell_{-1}(\heta)=\heta$ for
$\heta\in[0,\frac{1}{2}]$. In this case we obtain that
\[
\ell_{-1}(\heta) = \test{\heta\le\textstyle\frac{1}{2}}\heta +
\test{\heta\ge\textstyle\frac{1}{2}}\left(1-\log 2 -\heta-\log(1-\heta)\right).
\]

\section{Composite Losses}
\label{sec:composite-losses}
General loss functions are often constructed with the aid of a 
\emph{link function}. 
For a particular set of prediction values $\Vcal$ this is any continuous
mapping $\psi \colon [0,1] \to \Vcal$.
In this paper, our focus will be \emph{composite losses} for binary class 
probability estimation.
These are the composition of a CPE loss 
$\ell \colon \{-1,1\} \times [0,1] \to \RR$
and the inverse of a \emph{link function} $\psi$, an invertible
mapping from the unit interval to some range of values. 
Unless stated otherwise we will assume $\psi \colon [0,1] \to \RR$.
We will denote a composite loss by
\begin{equation}
	\ell^\psi(y,v) := \ell(y, \psi^{-1}(v)).
	\label{eq:general-composite-loss}
\end{equation}
The classical motivation for link functions \citep{McCullaghNelder1989}
is that often in estimating
$\eta$ one uses a parametric representation of
$\heta\colon\Xcal\rightarrow$[0,1] which has a natural scale not matching
$[0,1]$.  Traditionally one writes
$
\heta = \psi^{-1}(\hh)
$
where $\psi^{-1}$ is the ``inverse link'' (and $\psi$ is of course the forward
link). The function $\hh\colon\Xcal\rightarrow\reals$ is the 
{\em hypothesis}. Often $\hh=\hh_{\bm\alpha}$
is parametrised linearly in a parameter vector $\bm{\alpha}$. In such a 
situation it is
computationally convenient if $\ell(\eta,\psi^{-1}(\hh))$ is convex in $\hh$
(which implies it is convex in $\bm{\alpha}$ when $\hh_{\bm{\alpha}}$ is linear in 
$\bm{\alpha}$).  

Often one will choose the loss first (tailoring its properties
by the weighting given according to $w(c)$), and \emph{then} choose the link
somewhat arbitrarily to map the hypotheses appropriately. An interesting
alternative perspective arises in the literature on ``elicitability''.
\citet{Lambert:2008a}\footnote{See also \citep{Gneiting2009}.} provide a general characterisation of
proper scoring rules (i.e.~losses) for general \emph{properties} of 
distributions, that is,
continuous and locally non-constant functions $\Gamma$ which assign a real
value to each distribution over a finite sample space.  In the binary case,
these properties provide another interpretation of links that is complementary
to the usual one that treats the inverse link $\psi^{-1}$ as a way of
interpreting scores as class probabilities.

To see this, we first identify distributions over $\{-1,1\}$ with the 
probability $\eta$ of observing 1.
In this case properties are continuous, locally non-constant maps 
$\Gamma : [0,1] \to \RR$. 
When a link function $\psi$ is continuous it can therefore be interpreted as a 
property since its assumed invertibility implies it is locally non-constant.
A property $\Gamma$ is said to be \emph{elicitable} 
whenever there exists a strictly proper loss $\ell$ for it so that 
the composite loss $\ell^\Gamma$ satisfies for all $\heta \ne \eta$
\[
	L^\Gamma(\eta,\heta)
	:= \EE_{\Ysf\sim\eta}[\ell^\Gamma(\Ysf,\heta)] 
	> L^\Gamma(\eta,\eta).
\]

Theorem 1 of \citep{Lambert:2008a}  shows that $\Gamma$ is 
elicitable if and only if $\Gamma^{-1}(r)$ is convex for all 
$r\in\text{range}(\Gamma)$.
This immediately gives us a characterisation of ``proper'' link functions:
those that are both continuous and have convex level sets in $[0,1]$ --- they
are the non-decreasing continuous functions.   Thus in Lambert's perspective,
one chooses a ``property'' first (i.e. the invertible 
link) and \emph{then} chooses the
proper loss.


\subsection{Proper Composite Losses}
We will call a composite loss $\ell^\psi$ (\ref{eq:general-composite-loss}) 
a \emph{proper composite loss} if 
$\ell$ in (\ref{eq:general-composite-loss}) is a proper loss for class probability estimation.
As in the case for losses for probability estimation, the requirement that a
composite loss be proper imposes some constraints on its partial losses.
Many of the results for proper losses carry over to composite losses with some
extra factors to account for the link function.

\begin{theorem}\label{thm:composite-partials}
	Let $\lambda = \ell^\psi$ be a composite loss with differentiable and 
	strictly monotone link $\psi$
	and suppose the partial losses $\lambda_{-1}(v)$ and $\lambda_1(v)$ 
	are both differentiable.
	Then $\lambda$ is a proper composite loss if and only if
	there exists a weight function $w:(0,1)\to\reals^+$ such that
	for all $\heta\in(0,1)$
	\begin{equation}\label{eq:composite-partials}
		\frac{-\lambda'_1(\psi(\heta))}{1-\heta}
		= \frac{\lambda'_{-1}(\psi(\heta))}{\heta}
		= \frac{w(\heta)}{\psi'(\heta)}
		=: \rho(\heta),
	\end{equation}
	where equality is in the $L_1$ sense.
	Furthermore, $\rho(\heta)\ge 0$ for all $\heta\in(0,1)$.
\end{theorem}
\begin{proof}
	This is a direct consequence of Theorem~\ref{thm:partials} for proper losses 
	for probability estimation and the chain rule applied to 
	$\ell_y(\heta) = \lambda_y(\psi(\heta))$.
	Since $\psi$ is assumed to be strictly monotonic we know $\psi' > 0$ and so,
	since $w\ge 0$ we have $\rho \ge 0$. 
\end{proof}

As we shall see, the ratio $\rho(\heta)$ is a key quantity in the analysis of 
proper composite losses.
For example, Corollary~\ref{cor:proper-derivative} has natural analogue in terms 
of $\rho$ that will be of use later. 
It is obtained by letting $\heta=\psi^{-1}(v)$ and using the chain rule.
\begin{corollary}\label{cor:composite-derivative}
    Suppose $\ell^\psi$ is a proper composite loss with conditional risk
	denoted $L^\psi$. Then
	\begin{equation}\label{eq:composite-derivative}
		\frac{\partial}{\partial v}L^\psi(\eta,v)
		= (\psi^{-1}(v)-\eta)\rho(\psi^{-1}(v)).
	\end{equation}
\end{corollary}
Loosely speaking then, $\rho$ is a ``co-ordinate free'' weight function for 
composite losses where the link function $\psi$ is interpreted as a mapping from 
arbitrary $v\in\Vcal$ to values which can be interpreted as probabilities.

Another immediate corollary of Theorem~\ref{thm:composite-partials} shows how 
properness is characterised by a particular relationship between the choice
of link function and the choice of partial composite losses.
\begin{corollary}\label{cor:composite-partials-link}
	Let $\lambda := \ell^\psi$ be a composite loss with differentiable partial
	losses $\lambda_1$ and $\lambda_{-1}$.
	Then $\ell^\psi$ is proper if and only if the link $\psi$ satisfies
	\begin{equation}\label{eq:link-partials}
		\psi^{-1}(v) 
		= \frac{\lambda'_{-1}(v)}{\lambda'_{-1}(v) - \lambda'_1(v)}, \
		\ \ \ \ \forall v\in\Vcal .
	\end{equation}
\end{corollary}
\begin{proof}
	Substituting $\heta=\psi^{-1}(v)$ into (\ref{eq:composite-partials})
	yields
	$-\psi^{-1}(v)\lambda'_1(v) = (1-\psi^{-1}(v))\lambda'_{-1}(v)$ and
	solving this for $\psi^{-1}(v)$ gives the result.
\end{proof}

These results give some insight into the ``degrees of freedom'' available when 
specifying proper composite losses.
Theorem~\ref{thm:composite-partials} shows that the partial losses are 
completely determined once the weight function $w$ and $\psi$ (up to an additive 
constant) is fixed.
Corollary~\ref{cor:composite-partials-link} shows that for a given link $\psi$ 
one can specify one of the partial losses $\lambda_y$ but then properness 
fixes the other partial loss $\lambda_{-y}$.
Similarly, given an arbitrary choice of the partial losses, equation 
\ref{eq:link-partials} gives the single link which will guarantee the overall 
loss is proper.

We see then that Corollary~\ref{cor:composite-partials-link} provides us with
a way of constructing a \emph{reference link} for arbitrary composite losses 
specified by their partial losses.
The reference link can be seen to satisfy
\[
    \psi(\eta) = \argmin_{v\in\RR} L^{\psi}(\eta,v)
\]
for $\eta\in(0,1)$
and thus \emph{calibrates} a given composite loss in the sense of \cite{Cohen:2004}.

We now briefly consider an application of the parametrisation of proper 
losses as a weight function and link.
In order to implement Stochastic Gradient Descent (SGD) algorithms one 
needs to compute the derivative of the loss with respect to predictions 
$v\in\RR$.
Letting $\heta(v) = \psi^{-1}(v)$ be the probability estimates associated with
the prediction $v$, we can use (\ref{eq:composite-derivative}) when 
$\eta \in\{0,1\}$ to obtain the update rules for positive and negative examples:
\begin{eqnarray}
    \frac{\partial}{\partial v} \ell_1^\psi(v)
    &=& 
    (\heta(v)-1)\, \rho(\heta(v)),\label{eq:derivative-ell1}\\
    \frac{\partial}{\partial v} \ell_{-1}^\psi(v)
    &=& 
    \heta(v)\,\rho(\heta(v)).\label{eq:derivative-ell-1}
\end{eqnarray}
Given an arbitrary weight function $w$ (which defines a proper loss via 
Corollary~\ref{cor:proper-derivative} and Theorem~\ref{thm:savage}) and link
$\psi$, the above equations show that one could implement SGD directly 
parametrised in terms of $\rho$  without needing to explicitly compute 
the partial losses themselves.

Finally, we make a note of an analogue of Corollary~\ref{cor:regret-bregman} for
composite losses.
It shows that the regret for an arbitrary composite loss is related to a Bregman
divergence via its link.
\begin{corollary}
    Let $\ell^\psi$ be a proper composite loss with invertible link. 
    Then for all $\eta,\heta\in (0,1)$,
    \begin{equation}
	\Delta L^{\psi}(\eta,v)= D_{-\minL}(\eta,\psi^{-1}(v)) .
	\label{eq:bregman-general-loss}
    \end{equation}
    \label{theorem:bregman-general-loss}
\end{corollary}
This corollary generalises the results due to
\citet{Zhang2004} and
\citet{Masnadi-Shirazi:2009} who considered only margin losses respectively
without and with links.

\subsection{Margin Losses}

The \emph{margin} associated with a real-valued prediction $v\in\RR$ and label
$y\in\{-1,1\}$ is the product $z=yv$.
Any function $\phi \colon \RR \to \RR^+$ can be used as a \emph{margin loss} by 
interpreting $\phi(yv)$ as the penalty for predicting $v$ for an instance with
label $y$.
Margin losses are inherently symmetric since $yv = (-y)(-v)$ and so the penalty
$\phi(yv)$ given for predicting $v$ when the label is $y$ is necessarily the 
same as the penalty for predicting $-v$ when the label is $-y$.
Margin losses have attracted a lot of attention \citep{Bartlett2000}
because of their central role in Support Vector Machines 
\citep{CortesVapnik1995}.
In this section we explore the relationship between these margin losses and the 
more general class of composite losses and, in particular, symmetric
composite losses.

Recall that a general composite loss is of the form 
$\ell^\psi(y,v) = \ell(y,\psi^{-1}(v))$ for a loss 
$\ell \colon \Ycal\times [0,1] \to [0,\infty)$ and an invertible link 
$\psi\colon\RR\to [0,1]$.
We would like to understand when margin losses can be understood
as losses suitable for probability estimation tasks. 
As discussed above, proper losses are a natural class of losses over $[0,1]$
for probability estimation so a natural question in this vein is the following:
given a margin loss $\phi$ can we choose a link $\psi$ so that there exists a
proper loss $\ell$ such that $\phi(yv) = \ell^\psi(y,v)$?
In this case the proper loss will be $\ell(y,\heta) = \phi(y\psi(\heta))$.

The following corollary of Theorem~\ref{thm:composite-partials} gives necessary 
and sufficient conditions on the choice of link $\psi$ to guarantee when a 
margin loss $\phi$ can be expressed as a proper composite loss.
\begin{corollary}\label{cor:psi-from-phi}
	Suppose $\phi\colon\RR\to\RR$ is a differentiable margin loss.
	Then, $\phi(yv)$ can be expressed as a proper composite 
	loss $\ell^\psi(y,v)$ if and only if 
	the link $\psi$ satisfies
	\begin{equation}\label{eq:proper-phi-inverse}
		\psi^{-1}(v) = \frac{\phi'(-v)}{\phi'(-v)+\phi'(v)}.
	\end{equation}
\end{corollary}
\begin{proof}
	Margin losses, by definition, have partial losses
	$\lambda_y(v) = \phi(yv)$ which means
	$\lambda'_1(v) = \phi'(v)$ and $\lambda'_{-1}(v) = -\phi'(-v)$.
	Substituting these into (\ref{eq:link-partials}) gives the result.
\end{proof}

This result provides a way of interpreting predictions $v$ 
as probabilities $\heta = \psi^{-1}(v)$ in a 
consistent manner,
for a problem defined by a margin loss.
Conversely, it also guarantees that using any other link to interpret 
predictions as probabilities will be inconsistent.\footnote{
	Strictly speaking, if the margin loss has ``flat spots'' --- i.e., where 
	$\phi'(v) = 0$ --- then the choice of link may not be unique.
}
Another immediate implication is that for a margin loss to be considered a 
proper loss its link function must be \emph{symmetric} in the sense that
\[
	\psi^{-1}(-v)
	= \frac{\phi'(v)}{\phi'(v)+\phi'(-v)}
	= 1-\frac{\phi'(-v)}{\phi'(-v)+\phi'(v)}
	= 1-\psi^{-1}(v),
\]
and so, by letting $v=\psi(\heta)$, we have $\psi(1-\heta)=-\psi(\heta)$ and 
thus $\psi(\half) = 0$.

Corollary~\ref{cor:psi-from-phi} can also be seen as a simplified and 
generalised version of the argument by \citet{Masnadi-Shirazi:2009} 
that a concave minimal conditional risk function and a symmetric link completely 
determines a margin loss\footnote{
	\citet[Section 4.4]{Shen2005} seems to have been the first to view margin
	losses from this more general perspective.}.

We now consider a couple of specific margin losses and show how they can be
associated with a proper loss through the choice of link given in 
Corollary~\ref{cor:psi-from-phi}.
The exponential loss $\phi(v) = e^{-v}$ gives rise to a proper loss
$\ell(y,\heta) = \phi(y\psi(\heta))$ via the link 
\[
	\psi^{-1}(v) = \frac{-e^v}{-e^v - e^{-v}} = \frac{1}{1+e^{-2v}}
\]
which has non-zero denominator.
In this case $\psi(\heta) = \frac{1}{2}\log\left(\frac{\heta}{1-\heta}\right)$ 
is just the logistic link.
Now consider the family of margin losses parametrised by $\alpha\in(0,\infty)$
\[
	\phi_\alpha(v)=\frac{\log( \exp(1-v)\alpha)+1)}{\alpha}.
\]
This family of differentiable convex losses approximates the hinge loss as $\alpha\to 0$ and was studied in the multiclass case by \citet{Zhang:2009}.
Since these are all differentiable functions with 
$\phi'_\alpha(v) = \frac{-e^{\alpha(1-v)}}{e^{\alpha(1-v)} + 1}$, 
Corollary~\ref{cor:psi-from-phi} and a little algebra gives
\[
	\psi^{-1}(v) 
	= \left[
		1+\frac{e^{2\alpha}+e^{\alpha(1-v)}}{e^{2\alpha}+e^{\alpha(1+v)}}
	\right]^{-1}.
\]
Examining this family of inverse links as $\alpha\to 0$ gives some insight into 
why the hinge loss is a surrogate for classification but not probability 
estimation.
When $\alpha \approx 0$ an estimate $\heta = \psi^{-1}(v) \approx \half$ for
all but very large $v\in\RR$.
That is, in the limit all probability estimates sit infinitesimally to the 
right or left of $\half$ depending on the sign of $v$.

\section{Classification Calibration and Proper Losses}
\label{section:classification-calibrated}

The notion of properness of a loss designed for class probability estimation is
a natural one. If one is only interested in classification (rather than
estimating probabilities) a weaker condition suffices. In this section we will
relate the weaker condition to properness.

\subsection{Classification Calibration for CPE Losses}
We begin by giving a definition of classification calibration 
for CPE losses (\emph{i.e.}, over
the unit interval $[0,1]$) and relate it to composite losses via a link. 
\begin{definition}
    We say a CPE loss $\ell$ is \emph{classification calibrated at $c\in(0,1)$} 
    and write
    $\ell$ is $\mathrm{CC}_c$ if the associated conditional risk $L$ satisfies
    \begin{equation}
	\forall\eta\ne c, \ \ 
	    \minL(\eta) < \inf_{\heta \colon (\heta-c)(\eta-c)\le 0}
	    L(\eta,\heta) .
	    \label{eq:proper-cc-condition}
	\end{equation}
\end{definition}
The expression constraining the infimum ensures that $\heta$ is on the opposite side 
of $c$ to $\eta$, or $\heta=c$. 

The condition $\mathrm{CC}_{\half}$ is equivalent to
what is called ``classification calibrated'' by
\citet{BartlettJordanMcAuliffe2006} and ``Fisher consistent for classification
problems'' by \citet{Lin2002a} although their definitions were only for margin
losses.

One might suspect that there is a connection 
between classification calibrated at $c$ and standard
Fisher consistency for class probability estimation losses. The following
theorem, which
captures the intuition behind the ``probing'' 
reduction~\citep{LangfordZadrozny2005}, characterises the situation.
\begin{theorem}
    A CPE loss $\ell$ is $\mathrm{CC}_c$ for all $c\in(0,1)$ if and only if $\ell$ is
    strictly proper.
    \label{theorem:cc-for-all-c}
\end{theorem}
\begin{proof}
    $L$ is $\mathrm{CC}_c$ for all $c\in(0,1)$ is equivalent to
    \begin{eqnarray*}
	& & \forall c\in(0,1),\ \forall\eta\ne c\ \left\{\begin{array}{ll}
	    \minL(\eta) <\inf_{\heta\ge c} L(\eta,\heta), \ \ &\eta <c\\
	    \minL(\eta) <\inf_{\heta\le c} L(\eta,\heta), \ \ &\eta >c
	    \end{array}\right.\\
	& \Leftrightarrow & \forall\eta\in(0,1), \ \forall c\ne\eta\ 
	\left\{\begin{array}{l}
	    \forall c>\eta, \ \minL(\eta)<\inf_{\heta\ge c} L(\eta,\heta)\\
	    \forall c<\eta, \ \minL(\eta)<\inf_{\heta\le c} L(\eta,\heta)
	\end{array}\right.\\
	& \Leftrightarrow & \forall\eta\in(0,1),\ \left\{\begin{array}{l}
	    \minL(\eta) < \inf_{\heta\ge c>\eta} L(\eta,\heta)\\
	    \minL(\eta) < \inf_{\heta\le c<\eta} L(\eta,\heta)
	\end{array}\right.\\
	&\Leftrightarrow & \forall\eta\in(0,1), \ \minL(\eta) <
	\inf_{(\heta>\eta)\mbox{\ or\ }(\heta<\eta)} L(\eta,\heta)\\
	&\Leftrightarrow& \forall\eta\in(0,1),\ \minL(\eta)< \inf_{\heta\ne\eta}
	L(\eta,\heta)
    \end{eqnarray*}
    which means $L$ is strictly proper.
\end{proof}

The following theorem is a generalisation of the characterisation  of 
$\mathrm{CC}_{\frac{1}{2}}$ for margin losses via  $\phi'(0)$ due to
\citet{BartlettJordanMcAuliffe2006}.

\begin{theorem}
    Suppose  $\ell$ is a loss and suppose that $\ell_1'$ and
    $\ell_{-1}'$ exist everywhere. Then for any $c\in(0,1)$ $\ell$ is 
    $\mathrm{CC}_c$ if and only if
    \begin{equation}
	\ell_{-1}'(c)>0 \ \ \mbox{and}\ \ \ell_1'(c)<0\ \ \mbox{and}\ \ 
	c{\ell_1'(c)} +(1-c){\ell_{-1}'(c)}=0.
    \label{eq:CCc-condition}
    \end{equation}
    \label{theorem:CCc-characterisation}
\end{theorem}
\begin{proof}
    Since $\ell_1'$ and $\ell_{-1}'$ are assumed to exist everywhere
    \[
    \frac{\partial}{\partial\heta}
    L(\eta,\heta)=\eta\ell_1'(\heta)+(1-\eta)\ell_{-1}'(\heta)
    \]
    exists for
    all $\heta$.
    $L$ is $\mathrm{CC}_c$ is equivalent to
    \begin{eqnarray}
     & & \left.\frac{\partial}{\partial\heta} L(\eta,\heta)\right|_{\heta=c}
	\left\{\begin{array}{ll}
	    >0, \ \ &\eta<c<\heta\\
	    <0, \ \ &\heta<c<\eta
	\end{array}\right.\nonumber\\
	&\Leftrightarrow& \left\{\begin{array}{ll}
	    \forall\eta<c,\ &\eta\ell_1'(c)+(1-\eta)\ell_{-1}'(c)>0\\
	    \forall\eta>c,\ &\eta\ell_1'(c)+(1-\eta)\ell_{-1}'(c)<0
	\end{array}\right.\label{eq:ell-pair-condition}\\
	&\Leftrightarrow& \begin{array}{l}
	    c\ell_1'(c)+(1-c)\ell_{-1}'(c)=0\\
	    \mbox{and\ }\ell_{-1}'(c)>0\ \mbox{and\ } \ell_1'(c)<0,
	\end{array}\label{eq:ell1-ell0-condition}
    \end{eqnarray}
    where we have used the fact that (\ref{eq:ell-pair-condition}) with
    $\eta=0$ and $\eta=1$ respectively substituted implies $\ell_{-1}'(c)>0$ and
    $\ell_1'(c)<0$.
\end{proof}

If $\ell$ is proper, then by evaluating (\ref{eq:proper-deriviate}) at $\eta=0$ 
and $\eta=1$ we obtain
$\ell_1'(\heta)= -w(\heta)(1-\heta)$ and
$\ell_{-1}'(\heta)=w(\heta)\heta$.
Thus (\ref{eq:ell1-ell0-condition})  implies $-w(c)(1-c)<0$ and $w(c)c>0$
which holds if and only if $w(c)\ne 0$. We have thus shown the following
corollary.
\begin{corollary}
    If $\ell$ is  proper with weight $w$, then  for any
    $c\in(0,1)$,
    \[
	w(c)\ne 0 \ \Leftrightarrow\ \ell\mbox{\ is\ }\mathrm{CC}_c.
    \]
\end{corollary}
The simple form of the weight function for the cost-sensitive
misclassification loss $\ell_{c_0}$ ($w(c)=\delta(c-c_0)$)
gives the following corollary (confer \citet{BartlettJordanMcAuliffe2006}):
\begin{corollary}
    $\ell_{c_0}$ is $\mathrm{CC}_c$ if and only if $c_0=c$.
\end{corollary}

\subsection{Calibration for Composite Losses}

The translation of the above results to general proper composite losses with
invertible differentiable link $\psi$ is straight forward. Condition 
(\ref{eq:proper-cc-condition}) becomes
\[
	\forall\eta\ne c, \ \ 
	\minL^\psi(\eta) < \inf_{v \colon (\psi^{-1}(v)-c)(\eta-c)\le 0}
	L^\psi(\eta,\psi^{-1}(v)) .
\]
Theorem  \ref{theorem:cc-for-all-c} then immediately gives:
\begin{corollary}
    A composite loss $\ell^\psi(\cdot,\cdot)=\ell(\cdot,\psi^{-1}(\cdot))$ with invertible and differentiable link $\psi$
    is $\mathrm{CC}_c$ for all $c\in(0,1)$ if and only if the associated proper
    loss $\ell$ is strictly proper.
    \label{corollary:cc-for-allc-composite}
\end{corollary}
Theorem \ref{theorem:CCc-characterisation} immediately gives:
\begin{corollary}
    Suppose $\ell^\psi$ is as in Corollary
    \ref{corollary:cc-for-allc-composite} and that the partial losses
    $\ell_{1}$ and $\ell_{-1}$ of the associated proper loss $\ell$ are
    differentiable. Then for any $c\in(0,1)$, $\ell^\psi$ is $\mathrm{CC}_c$ if
    and only if (\ref{eq:CCc-condition}) holds.
    \label{corollary:CCc-characterisation}
\end{corollary}
It can be shown that in the special case of margin losses $L_\phi$, which
satisfy the conditions of Corollary \ref{cor:psi-from-phi} such that they are
proper composite losses, Corollary \ref{corollary:CCc-characterisation} leads
to the condition $\phi'(0)<0$ which is the same as obtained by
\citet{BartlettJordanMcAuliffe2006}.

\section{Convexity of Composite Losses}
\label{section:convexity}

We have seen that composite losses are defined by the proper loss $\ell$ and
the link $\psi$. We have further seen from (\ref{cor:psi-from-phi}) that it is
natural to parametrise composite losses in terms of $w$ and $\psi'$, and
combine them as $\rho$. One may wish to choose a weight function $w$ and 
determine which
links $\psi$ lead to a convex loss; or choose a link $\psi$ and determine which
weight functions $w$ (and hence proper losses) lead to a convex composite loss.
The main result of this section is
Theorem~\ref{theorem:simple-characterisation-of-convexity} answers these
questions by  characterising the convexity of composite losses in terms 
of $(w,\psi')$ or $\rho$.

We first establish some convexity results for losses and their conditional and
full risks.

\begin{lemma}\label{lem:convexity}
	Let $\ell   : \Ycal\times\Vcal\to[0,\infty)$ 
	denote an arbitrary  loss.
	Then the following are equivalent:
	\begin{enumerate}
		\setlength{\itemsep}{-1mm}
		\item $v \mapsto \ell(y,v)$ is convex for all 
		      $y\in\{-1,1\}$,
			\label{item:point}
		\item $v \mapsto L(\eta,v)$ is convex for all 
		    $\eta\in[0,1]$,
			\label{item:conditional}
		\item $v \mapsto \hat{\LL}(v,S) 
			:= \frac{1}{|S|}\sum_{(x,y)\in S} \ell(y,v(x))$ 
			is convex for all finite $S \subset \Xcal\times\Ycal$.
			\label{item:full}
	\end{enumerate}
\end{lemma}
\begin{proof}
	\ref{item:point} $\Rightarrow$ \ref{item:conditional}: 
	    By definition, $L(\eta,v)=
	    (1-\eta)\ell(-1,v)+\eta\ell(1,v)$ 
		which is just a convex combination of convex functions
		and hence convex.

	\ref{item:conditional} $\Rightarrow$ \ref{item:point}: 
		Choose $\eta = 0$ and $\eta = 1$ in the definition of $L$.

	\ref{item:point} $\Rightarrow$ \ref{item:full}:
		For a fixed $(x,y)$, the function $v \mapsto \ell(y,v(x))$ is 
		convex since $\ell$ is convex. Thus, $\hat{\LL}$ is convex as it is a 
		non-negative weighted sum of convex functions.

	\ref{item:full} $\Rightarrow$ \ref{item:point}:
		The convexity of $\hat{\LL}$ holds for every $S$ so for each 
		$y\in\{-1,1\}$ choose $S = \{ (x,y) \}$ for some $x$. In each case 
		$v \mapsto \hat{\LL}(v,S) = \ell(y,v(x))$ is convex as
		required.  
\end{proof}
The following theorem generalises the corollary on page 12 of \citet{Buja:2005}
to arbitrary composite losses with invertible links. 
It has less practical value than the previous lemma since, in general, 
sums of 
quasi-convex functions are not necessarily quasi-convex (a function $f$ is 
quasi-convex if the set $\{x\colon f(x)\ge\alpha\}$ is convex
for all $\alpha\in\reals$). 
Thus, assuming properness of the loss $\ell$ does not guarantee its empirical 
risk $\hat{\LL}(\cdot,S)$ will not have local minima.
\begin{theorem}
	If $\ell^{\psi}(y,v) = \ell(y,\psi^{-1}(v))$ is a composite loss where
	$\ell$ is proper and $\psi$ is invertible and differentiable then 
	$L^{\psi}(\eta,v)$ is quasi-convex in $v$ for all $\eta\in[0,1]$.
	\label{theorem:quasi-convex}
\end{theorem}
\begin{proof}
    Since $\ell$ is proper we know by Corollary~\ref{cor:composite-derivative} 
	that the conditional Bayes risk satisfies
	\begin{eqnarray*}
		\frac{\partial}{\partial v} L^{\psi}(\eta,v) 
		& = &
		(\psi^{-1}(v) - \eta)\rho(\psi^{-1}(v)).
	\end{eqnarray*}
	Since $\psi$ is invertible and $\rho \ge 0$ we see that 
	$\frac{\partial}{\partial v} L^{\psi}(\eta,v)$ only changes sign
	at $\eta = \psi^{-1}(v)$ and so $L^{\psi}$ is quasi-convex as required.
\end{proof}

%

The following theorem characterises convexity of composite losses with
invertible links.

\begin{theorem}
    \label{theorem:general-characterisation}
    Let $\ell^\psi(y,v)$  be a composite loss comprising an invertible link
    $\psi$ with inverse $q:=\psi^{-1}$ and strictly proper loss with weight 
    function $w$.  Assume $q'(\cdot)>0$.
    Then $v\mapsto\ell^\psi(y,v)$  is convex for $y\in\{-1,1\}$ if and only if
\begin{equation}
    -\frac{1}{x}\ \le\ 
    \frac{w'(x)}{w(x)} - \frac{\psi''(x)}{\psi'(x)}
    \ \le\ \frac{1}{1-x} ,\ \ \ \ 
    \forall x\in (0,1).
    \label{eq:general-characterisation}
\end{equation}
\end{theorem}
This theorem suggests a very natural parametrisation of composite losses is
via $(w,\psi')$. Observe that $w,\psi'\colon [0,1]\rightarrow\reals^+$. (But
also see the comment following 
Theorem~\ref{theorem:simple-characterisation-of-convexity}.)
\begin{proof}
We can write the conditional composite loss as
\[
    L^\psi(\eta,v)=\eta\ell_1(q(v)) + (1-\eta)\ell_{-1}(q(v))
\]
and by substituting $q=\psi^{-1}$ into (\ref{eq:composite-derivative}) we have
\begin{eqnarray}
    \frac{\partial}{\partial v} L^\psi(\eta,v) 
    &=& w(q(v))q'(v)[q(v)-\eta]\label{eq:first-derivative-composite-loss}.
\end{eqnarray}
A necessary and sufficient condition for $v\mapsto \ell^\psi(y,v)=L^\psi(y,v)$
to be convex for $y\in\{-1,1\}$ is that 
\[
    \frac{\partial^2}{\partial v^2} L^\psi(y,v) \ge 0, \ \ \ \ \forall
    v\in\reals,\ \forall {y}\in\{-1,1\}.
\]
Using (\ref{eq:first-derivative-composite-loss}) the above condition is
equivalent to
\begin{equation}
    [w(q(v))q'(v)]' (q(v)-\test{y=1}) + w(q(v)) q'(v) q'(v)\ \ge\ 0, 
    \ \ \ \forall v\in\reals,
    \label{eq:complex-derivative}
\end{equation}
where 
\[
[w(q(v))q'(v)]' := \frac{\partial}{\partial v} w(q(v))q'(v).
\]
Inequality (\ref{eq:complex-derivative}) is equivalent to 
\cite[equation 39]{Buja:2005}. 
By further manipulations, we can simplify
(\ref{eq:complex-derivative}) considerably.

Since $\test{y=1}$ is either 0 or 1 we equivalently have the two inequalities
\begin{eqnarray*}
    {}[w(q(v))q'(v)]' q(v) +w(q(v))(q'(v))^2 &\ge & 0, \ \ \ \forall v\in\reals,\
    \ \ ({y}=-1)\\
    {}[w(q(v))q'(v)]'(q(v)-1) +w(q(v))(q'(v))^2 &\ge & 0,\ \ \ \forall
    v\in\reals,\ \ \ ({y}=1),
\end{eqnarray*}
which we shall rewrite as the pair of inequalities
\begin{eqnarray}
    w(q(v))(q'(v))^2 &\ge& -q(v)[w(q(v))q'(v)]',\ \ \ \ \ \ \ \ \forall v\in\reals, \ \ \
     \label{eq:inequality-case-0}\\
    w(q(v))(q'(v))^2 &\ge& (1-q(v))[w(q(v))q'(v)]',\ \ \ \forall v\in\reals.
 \label{eq:inequality-case-1}
\end{eqnarray}
Observe that if  $q(\cdot)=0$ (resp.~$1-q(\cdot)=0$) then
(\ref{eq:inequality-case-0}) (resp.~(\ref{eq:inequality-case-1})) is satisfied
anyway because of the assumption on $q'$ and the fact that $w$ is
non-negative. It is thus equivalent to restrict consideration to $v$ in the set
\[
    \{x\colon q(x)\ne0 \mbox{\ \ and\ \ }(1-q(x))\ne 0\} = 
    q^{-1}( (0,1) )= \psi( (0,1) ).
\]
Combining (\ref{eq:inequality-case-0}) and (\ref{eq:inequality-case-1}) 
we obtain the equivalent condition
\begin{equation}
     \frac{(q'(v))^2}{1-q(v)}\  \ge\ 
    \frac{[w(q(v))q'(v)]'}{w(q(v))}\ \ge\ 
    \frac{-(q'(v))^2}{q(v)},\ \ \ \forall v\in\psi( (0,1)),
    \label{eq:two-inequalities-as-one}
\end{equation}
where we have used the fact that $q\colon\reals\rightarrow [0,1]$ and is 
thus sign-definite and
consequently $-q(\cdot)$ is always negative
and division by $q(v)$ and $1-q(v)$ is permissible  since as argued we can
neglect the cases when these take on the value zero,
and division by $w(q(v))$ is permissible by the assumption
of \emph{strict} properness since that implies $w(\cdot)>0$. 
Now
\[
[w(q(\cdot))q'(\cdot)]' = w'(q(\cdot))q'(\cdot)q'(\cdot)+w(q(\cdot))q''(\cdot)
\]
and thus (\ref{eq:two-inequalities-as-one}) is equivalent to
\begin{equation}
     \frac{(q'(v))^2}{1-q(v)}\ \ge\ 
      \frac{w'(q(v))(q'(v))^2+w(q(v))q''(v)}{w(q(v))} \ \ge \ 
      \frac{-(q'(v))^2}{q(v)}, \ \ \ \forall v\in \psi( (0,1) )
      \label{eq:double-inequality-modified}
  \end{equation}
Now divide all sides of (\ref{eq:double-inequality-modified})
by $(q'(\cdot))^2$ (which is permissible by assumption). 
This gives the equivalent condition
\begin{equation}
    \frac{1}{1-q(v)}\  \ge\  \frac{w'(q(v))}{w(q(v))} +
    \frac{q''(v)}{(q'(v))^2}\  \ge\  \frac{-1}{q(v)} ,
     \ \ \ \forall v\in\psi( (0,1) ).
\label{eq:a}
\end{equation}
Let $x=q(v)$ and so $v=q^{-1}(x)=\psi(x)$. Then (\ref{eq:a}) is equivalent to
\begin{equation}
    \frac{1}{1-x} \ \ge\  
    \frac{w'(x)}{w(x)} + \frac{q''(\psi(x))}{(q'(\psi(x)))^2}
    \ \ge\ \frac{-1}{x},\ \ \ \forall x\in (0,1).
    \label{eq:a2}
\end{equation}
Now $\frac{1}{q'(\psi(x))}=\frac{1}{q'(q^{-1}(x))}= (q^{-1})'(x)=\psi'(x)$. 
Thus (\ref{eq:a2}) is equivalent to
\begin{equation}
    \frac{1}{1-x}\ \ge\ \frac{w'(x)}{w(x)} + \Phi_\psi(x)
    \ \ge\ \frac{-1}{x},\ \ \ \forall x\in (0,1), 
    \label{eq:simpler-combined-inequalities}
\end{equation}
where 
\begin{equation}
    \Phi_\psi(x):= q''\left(\psi(x)\right)\left(\psi'(x)\right)^2 .
    \label{eq:Phi-def}
\end{equation}
All of the above steps are equivalences.
We have thus shown that 
\[
     \mbox{(\ref{eq:simpler-combined-inequalities}) is true\ } 
     \  \Leftrightarrow \ \mbox{$v\mapsto L^\psi(y,v)$ is convex for\ }
     y\in\{-1,1\}
\]
where the right hand side is equivalent to the assertion in the 
theorem by Lemma~\ref{lem:convexity}.

Finally we simplify $\Phi_\psi$.  We first
compute $q''$ in terms of $\psi=q^{-1}$. Observe that 
$q'=(\psi^{-1})'=\frac{1}{\psi'(\psi^{-1}(\cdot))}$. Thus
\begin{eqnarray*}
    q''(\cdot) &=& (\psi^{-1})''(\cdot)\\
    &=& \left(\frac{1}{\psi'(\psi^{-1}(\cdot))}\right)'\\
    &=& \frac{-1}{\left(\psi'(\psi^{-1}(\cdot))\right)^2} 
      \psi''(\psi^{-1}(\cdot)) \left(\psi^{-1}(\cdot)\right)'\\
    &=& \frac{-1}{\left(\psi'(\psi^{-1}(\cdot))\right)^3}
       \psi''(\psi^{-1}(\cdot)).
\end{eqnarray*}
Thus by substitution
\begin{eqnarray}
    \Phi_\psi(\cdot) &=& \frac{-1}{\left(\psi'(\psi^{-1}(\psi(\cdot)))\right)^3}
    \psi''(\psi(\psi^{-1}(\cdot)))\left(\psi'(\cdot)\right)^2\nonumber\\
    &=&\frac{-1}{\left(\psi'(\cdot)\right)^3}
    \psi''(\cdot)\left(\psi'(\cdot)\right)^2\nonumber\\
    &=& -\frac{\psi''(\cdot)}{\psi'(\cdot)}\label{eq:simple-Phi-psi} .
\end{eqnarray}
Substituting the simpler expression (\ref{eq:simple-Phi-psi}) 
for $\Phi_\psi$  into (\ref{eq:simpler-combined-inequalities}) completes 
the proof.
\end{proof}
\begin{lemma}
    If $q$ is affine  then $\Phi_\psi=0$.
    \label{lemma:phi-zero-q-linear}
\end{lemma}
\begin{proof}
    Using  (\ref{eq:simple-Phi-psi}), this is immediate
    since in this case $\psi''(\cdot)=0$.
\end{proof}
\begin{corollary}
    Composite losses with a linear link (including as a special case the identity
    link) are convex if and only if
    \[
    -\frac{1}{x}\  \le\ \frac{w'(x)}{w(x)}\ \le\ \frac{1}{1-x},\ \ \ 
    \forall x\in (0,1).
    \]
\end{corollary}

\subsection{Canonical Links}
\citet{Buja:2005} introduced the notion of a \emph{canonical link} defined by
$\psi'(v)=w(v)$.  
The canonical link corresponds to the notion of ``matching loss'' 
as developed by \citet{HelmboldKivinenWarmuth1999} and
\citet{KivinenWarmuth2001}.
Note that choice of canonical link implies $\rho(c)=w(c)/\psi'(c) = 1$.

\begin{lemma}
    Suppose $\ell$ is a proper loss with weight function $w$ and
    $\psi$ is the corresponding canonical link, then
    \begin{equation}
	\Phi_\psi(x)=-\frac{w'(x)}{w(x)} .
	\label{eq:Phi-canonical}
    \end{equation}
    \label{theorem:composite-canonical}
\end{lemma}
\begin{proof}
    Substitute $\psi'=w$ into (\ref{eq:simple-Phi-psi}).
\end{proof}
This lemma gives an immediate proof of the following result due to
\citet{Buja:2005}.
\begin{theorem}
    A composite loss comprising a proper loss with weight function $w$ 
    combined with its canonical link is always convex.
\end{theorem}
\begin{proof}
Substitute (\ref{eq:Phi-canonical}) into
    (\ref{eq:general-characterisation}) to obtain
    \[
    -\frac{1}{x} \ \le\ 0\ \le\ \frac{1}{1-x}, \ \ \ \forall x\in (0,1)
    \]
    which holds for any $w$.
\end{proof}
An alternative view of canonical links is given in
Appendix~\ref{section:appendix-canonical}.


\subsection{A Simpler Characterisation of Convex Composite Losses}

The following theorem prrovides a simpler characterisation of the convexity of
composite losses.
Noting that loss functions can be multiplied by a scalar
without affecting what a learning algorithm will do, it is convenient to
normalise them.
If $w$ satisfies
    (\ref{eq:general-characterisation}) then so does $\alpha w$ for all
    $\alpha\in(0,\infty)$. Thus without loss of generality we will normalise $w$
    such that $w(\frac{1}{2})=1$.  We chose to normalise about $\frac{1}{2}$
    for two reasons:
    symmetry and the fact that $w$ can have non-integrable singularities at $0$
    and $1$; see e.g.~\cite{Buja:2005}.
\begin{theorem}
    Consider a proper composite loss $\ell^\psi$
     with invertible link $\psi$ and (strictly proper) weight $w$ normalised such that
      $w(\half)=1$. Then $\ell$ is convex if and only if
    \begin{equation}
	\frac{\psi'(x)}{x} \ \lesseqgtr\ 2\psi'({\textstyle\frac{1}{2}})\, w(x)
	\ \lesseqgtr\ 
	\frac{\psi'(x)}{1-x},\ \ \  \forall x\in(0,1),
	\label{eq:simpler-characterisation}
    \end{equation}
    where $\lesseqgtr$ denotes $\le$ for $x\ge \frac{1}{2}$ and denotes 
    $\ge$ for $x\le \frac{1}{2}$.
    \label{theorem:simple-characterisation-of-convexity}
\end{theorem}
Observe that the condition (\ref{eq:simpler-characterisation}) is equivalent to
    \begin{equation}
	\frac{1}{2\psi'({\textstyle\frac{1}{2}})x} \ \lesseqgtr\ \rho(x)
	\ \lesseqgtr\ 
	\frac{1}{2\psi'({\textstyle\frac{1}{2}})(1-x)},\ \ \  \forall x\in(0,1),
    \end{equation}
    which suggests the importance of the function $\rho(\cdot)$.
\begin{proof}
     Observing that
    $\frac{w'(x)}{w(x)}=(\log w)'(x)$ we let $g(x):=\log w(x)$. Observe that 
    $g(v)=\int_{\frac{1}{2}}^v g'(x) dx +g(\frac{1}{2})$
    and $g(\frac{1}{2})=\log w(\frac{1}{2})=0$. Thus from 
    (\ref{eq:general-characterisation})
    we obtain
    \[
    -\frac{1}{x}-\Phi_\psi(x) \ \le\  g'(x)\ \le\ \frac{1}{1-x}-\Phi_\psi(x).
    \]
    For $v\ge \half$ we thus have
    \[
    \int_{\half}^v -\frac{1}{x}-\Phi_\psi(x) dx\ \le\ g(v)\ \le \ 
    \int_{\half}^v \frac{1}{1-x}-\Phi_\psi(x) dx.
    \]
    Conversely, for $v\le \half$ we have
    \[
    \int_{\half}^v -\frac{1}{x}-\Phi_\psi(x) dx \  \ge\ g(v)\ \ge  \ 
    \int_{\half}^v \frac{1}{1-x}-\Phi_\psi(x) dx,
    \]
    and thus
    \[
    -\ln v -\ln 2 - \int_{\half}^v \Phi_\psi(x) dx \ \lesseqgtr\ g(v)\ 
    \lesseqgtr\ -\ln
    2 -\ln(1-v) -\int_{\half}^v \Phi_\psi(x) dx .
    \]
    Since $\exp(\cdot)$ is monotone increasing we can apply it to all terms
    and obtain
    \begin{equation}
    \frac{1}{2v} \exp\left(-\int_{\half}^v \Phi_\psi(x) dx\right)\ \lesseqgtr\ 
    w(v)\ 
    \lesseqgtr\ \frac{1}{2(1-v)}\exp\left(-\int_{\half}^v \Phi_\psi(x) dx\right) .
    \label{eq:double-inequality}
\end{equation}
Now 
\[
    \int_{\frac{1}{2}}^v \Phi_\psi(x) dv = \int_{\frac{1}{2}}^v
    -\frac{\psi''(x)}{\psi'(x)} dx =-\int_{\half}^v (\log\psi')'(x) dx =
    -\log\psi'(v)+\log\psi'({\textstyle\frac{1}{2}})
\]
and so
\[
\exp\left(-\int_{\half}^v \Phi_\psi(x) dx\right) =
\frac{\psi'(v)}{\psi'(\frac{1}{2})}.
\]
Substituting into (\ref{eq:double-inequality}) completes the proof.
\end{proof}
\begin{figure}
    \begin{center}
	\includegraphics[width=0.45\textwidth]{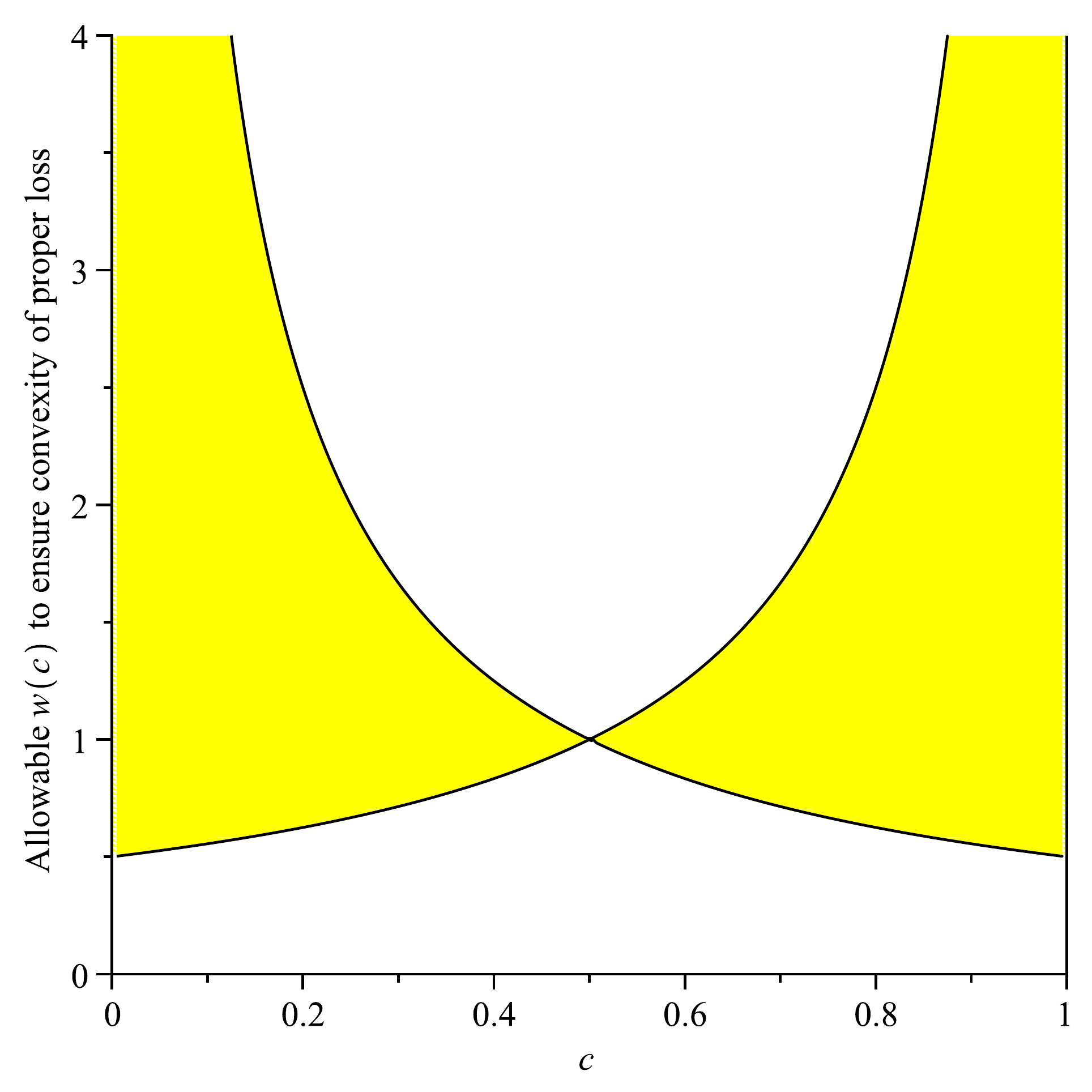}\hspace*{5mm}
    \includegraphics[width=0.45\textwidth]{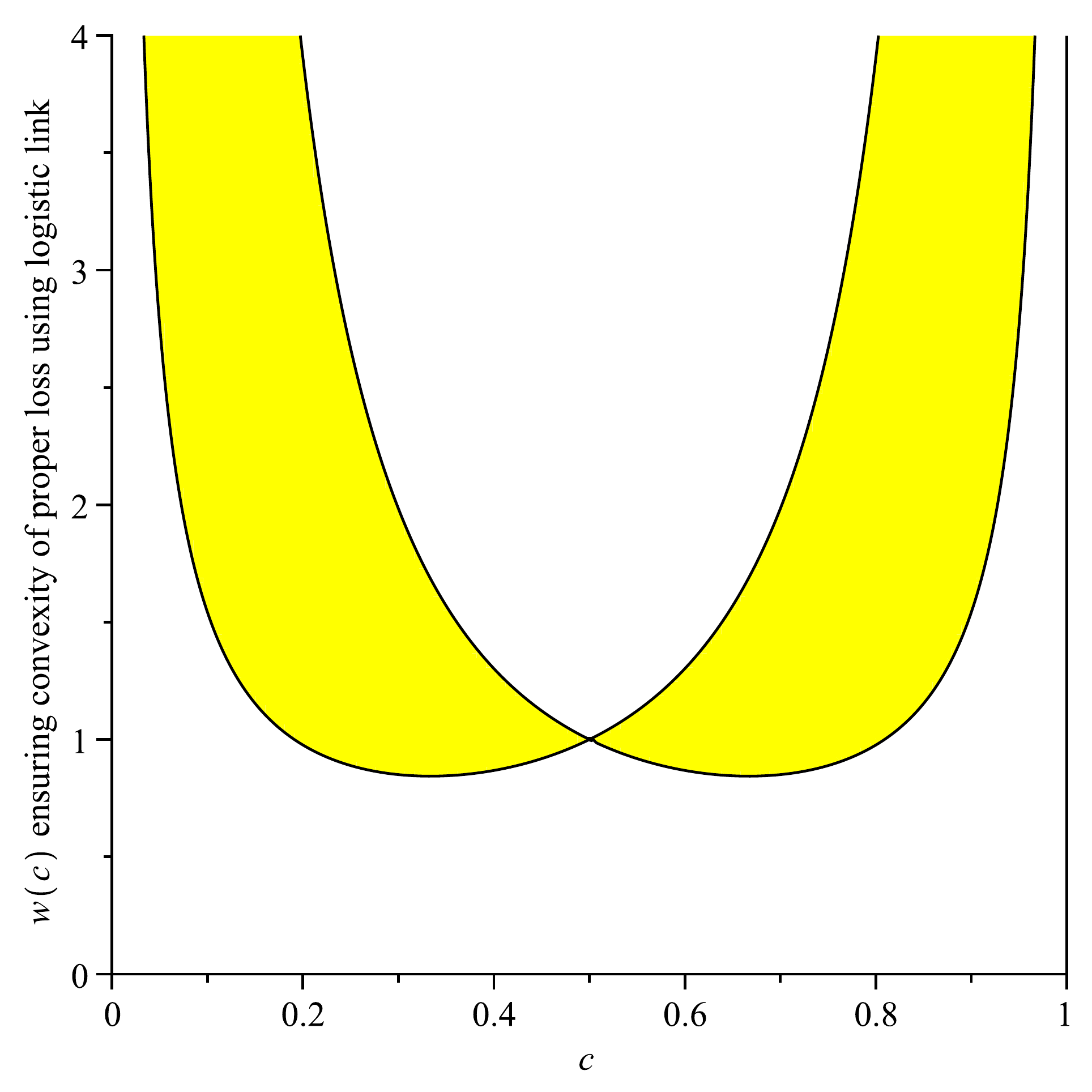}
\end{center}
	\caption{Allowable normalised weight functions to ensure convexity of
	composite loss functions with identity link (left) and logistic link (right).\label{fig:convexity-region}}
\end{figure}
If $\psi$ is the identity (\emph{i.e.} if $\ell^\psi$ is itself proper) 
we get the simpler constraints
\begin{equation}
    \frac{1}{2x}\ \lesseqgtr\ w(x)\ \lesseqgtr\ \frac{1}{2(1-x)}, \ \ \ 
    \forall x\in(0,1),
    \label{eq:double-inequality-simple}
\end{equation}
which are illustrated as the shaded region in Figure \ref{fig:convexity-region}.
Observe that the (normalised) weight function for squared loss is $w(c)=1$
which is indeed within the shaded region as one would expect.

Consider the link 
$\psi^{\mathrm{logit}}(c):=\log\left(\frac{c}{1-c}\right)$ 
with corresponding inverse link $q(c)=\frac{1}{1+e^{-c}}$. One can check that
$\psi'(c)=\frac{1}{c(1-c)}$.
Thus the constraints on the weight function $w$ 
to ensure convexity of the composite loss are
\[
    \frac{1}{8x^2(1-x)} \ \lesseqgtr\ w(x)\ \lesseqgtr\ \frac{1}{8x(1-x)^2},\ \
    \ \ \forall x\in (0,1).
\]
This is shown graphically in Figure~\ref{fig:convexity-region}.
One can compute similar regions for any link. Two other examples are
the Complementary Log-Log link $\psi^{\mathrm{CLL}}(x)=\log(-\log(1-x))$
(confer \citet{McCullaghNelder1989}), the ``square link''
$\psi^{\mathrm{sq}}(x)=x^2$ and the ``cosine link''
$\psi^{\mathrm{cos}}(x)=1-\cos(\pi x)$. All of these are illustrated in
Figure~\ref{fig:convexity-region-wierd-links}.
\begin{figure}
    \begin{center}
    \includegraphics[width=0.31\textwidth]{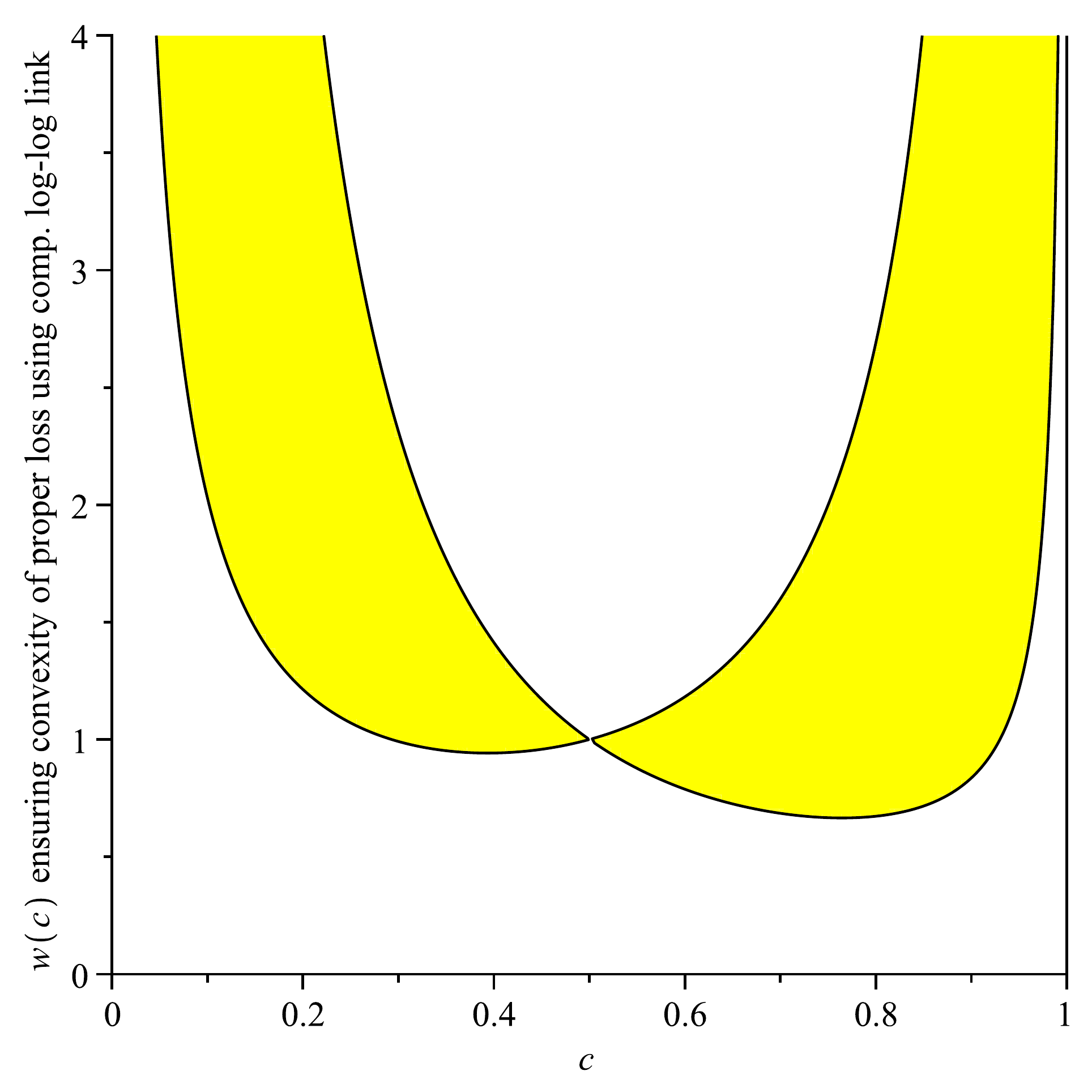}
    \hspace*{1mm}
    \includegraphics[width=0.31\textwidth]{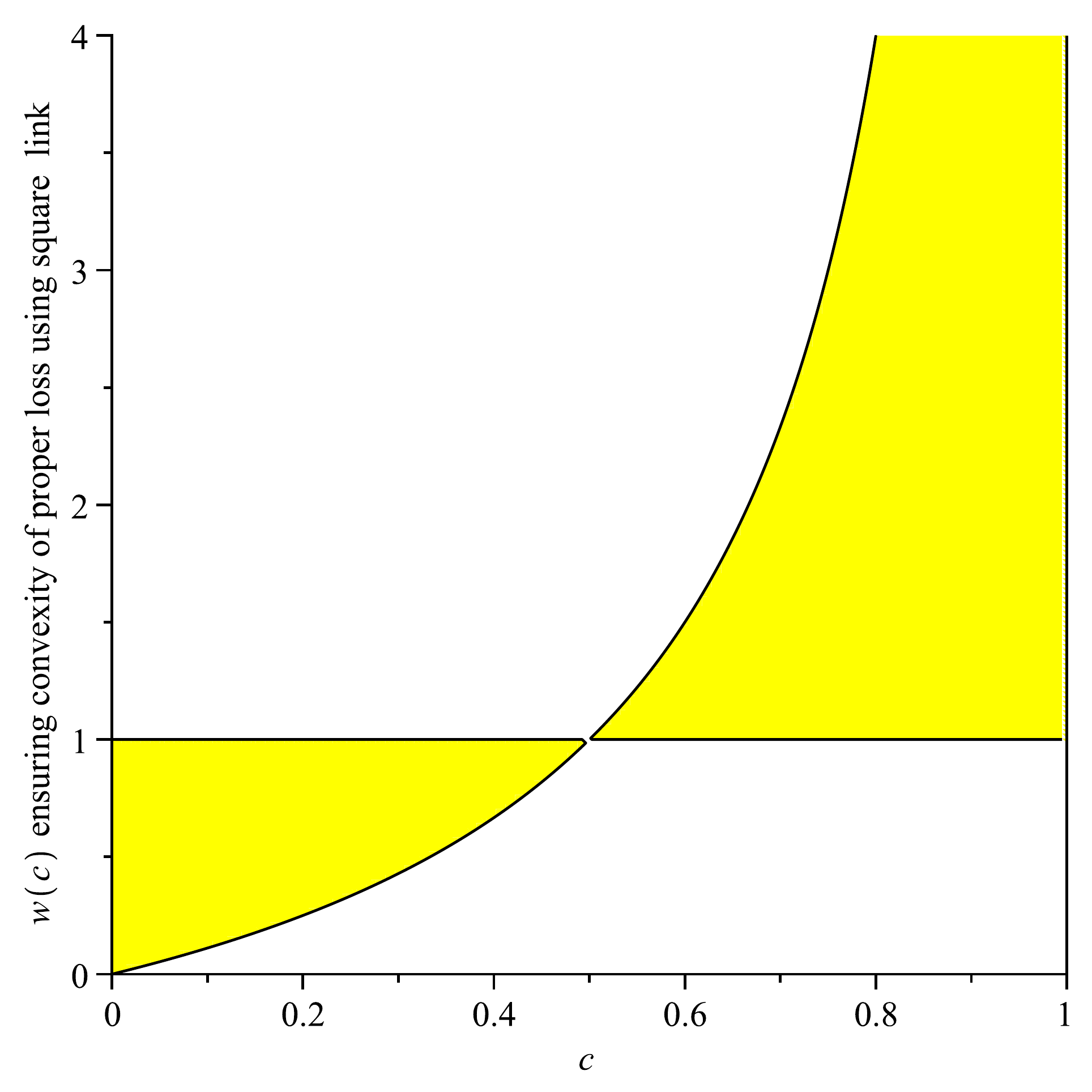}
    \hspace*{1mm}
    \includegraphics[width=0.31\textwidth]{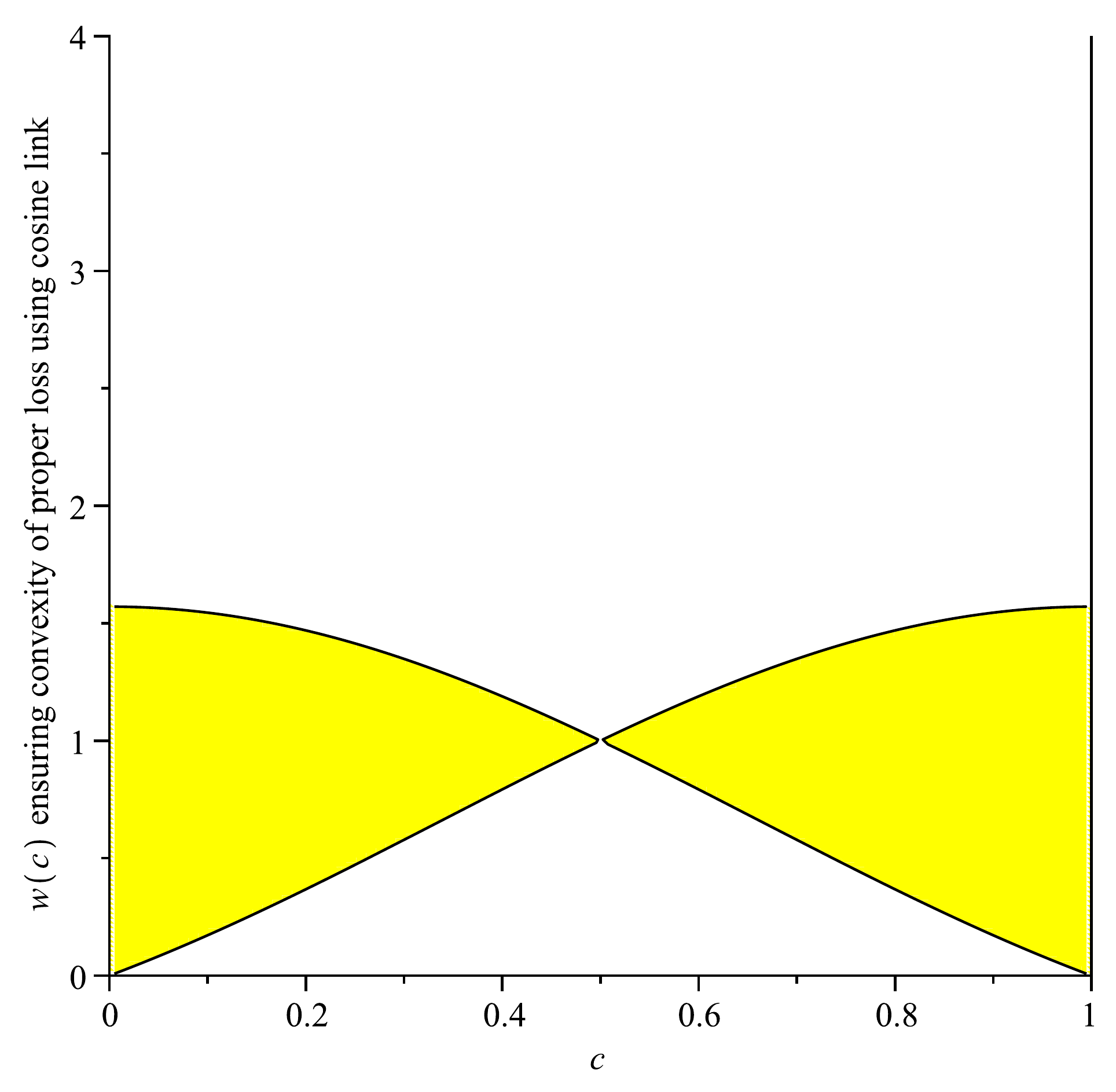}
\end{center}
	\caption{Allowable normalised weight functions to ensure convexity of
	loss functions with complementary log-log, square and cosine links.
	\label{fig:convexity-region-wierd-links}}
\end{figure}
The reason for considering these last two rather unusual links is to 
illustrate the following fact. Observing that the allowable region in
Figure~\ref{fig:convexity-region} precludes weight functions that approach zero
at the endpoints of the interval, and noting that in order to well approximate
the behaviour of 0-1 loss (with its weight function being
$w_{0-1}(c)=\delta(c-\frac{1}{2})$) one would like a weight function that does
indeed approach zero at the end points, it is natural to ask what constraints 
are imposed upon a link $\psi$ such that a composite loss with that link and a
weight function $w(c)$ such that 
\begin{equation}
\lim_{c\searrow 0} w(c) = \lim_{c\nearrow 1} w(c) =0
\label{eq:weights-at-endpoints}
\end{equation}
is convex.
Inspection of (\ref{eq:simpler-characterisation}) reveals it is
necessary that $\psi'(x)\rightarrow 0$ as $x\rightarrow 0$ and $x\rightarrow
1$. Such $\psi$ necessarily have bounded range and thus the inverse link
$\psi^{-1}$ is only defined on a finite interval and furthermore the gradient
of $\psi^{-1}$ will be arbitrarily large. If one wants inverse links
defined on the whole real line (such as the logistic link) then one can not
obtain a convex composite link with the associated proper loss having a weight
function satisfying (\ref{eq:weights-at-endpoints}). Thus one can not choose an
effectively usable link to ensure convexity of a proper loss that is
arbitrarily ``close
to'' 0-1 loss in the sense of the corresponding weight functions.

\begin{corollary}\label{cor:convex-proper}
    If a loss is proper and convex, then it is strictly proper.
    \label{corollary:proper-convex-strictly-proper}
\end{corollary}
The proof of Corollary
\ref{corollary:proper-convex-strictly-proper}
makes use of the following special case of the Gronwall style 
Lemma 1.1.1 of \citet{Bainov1992}.
\begin{lemma}
    Let $b\colon\reals\rightarrow\reals$ be continuous for $t\ge\alpha$. Let
    $v(t)$ be differentiable for $t\ge\alpha$ and suppose
    $ v'(t)\le b(t)v(t)$, for $t\ge\alpha $ and
	$v(\alpha)\le v_0$.
    Then for $t\ge\alpha$,
    \[
	v(t)\le v_0 \exp\left(\int_\alpha^t b(s)ds\right).
    \]
\end{lemma}
\begin{proof}{\bf (Corollary \ref{corollary:proper-convex-strictly-proper})}
    Observe that the RHS of (\ref{eq:general-characterisation}) implies
    \[
    w'(v)\le \frac{w(v)}{1-v}, \  v\ge 0.
    \]
    Suppose $w(0)=0$. Then $v_0=0$ and the setting $\alpha=0$ the lemma implies
    \[
    w(t)\le v_0 \exp\left(\int_0^t \frac{1}{1-s} ds\right) = \frac{v_0}{1-t}=0,
    \ \ \ \ t\in(0,1].
    \]
    Thus if $w(0)=0$ then $w(t)=0$ for all $t\in(0,1)$. Choosing
    any other $\alpha\in(0,1)$ leads to a similar conclusion. Thus if $w(t)=0$
    for some $t\in[0,1)$, $w(s)=0$ for all $s\in[t,1]$. Hence $w(t)>0$ for all
    $t\in[0,1]$ and hence by the remark immediately following Theorem
    \ref{thm:schervish}
    $\ell$ is strictly proper.
\end{proof}

\section{Choosing a Surrogate Loss}
\label{section:choosing}
A \emph{surrogate} loss function is a loss function which is not exactly what 
one wishes to minimise but is easier to work with algorithmically.  
Convex surrogate losses are often used in place of the 0-1 loss which is not 
convex.

Surrogate losses have garnered increasing interest in the machine learning
community \citep{Zhang2004, BartlettJordanMcAuliffe2006, Steinwart:2007,
Steinwart2008}.  
Some of the questions considered to date are bounding the
regret of a desired loss in terms of a surrogate (``surrogate regret bounds''
--- see \citep{ReidWilliamson2009} and references therein), the 
relationship between the
decision theoretic perspective and the elicitability perspective
\citep{Masnadi-Shirazi:2009}, and efficient algorithms for minimising convex
surrogate margin losses~\citep{Nock:2008,Nock2009}.

Typically convex surrogates are used because they lead to convex, and thus
tractable, optimisation problems. To date, work on surrogate losses has focussed
on margin losses which necessarily are symmetric with respect to false
positives and false negatives~\citep{Buja:2005}. In line with the rest of this
paper, our treatment will not be so restricted.

\subsection{The ``Best'' Surrogate Loss}
\label{section:best-surrogate}
There are many choices of surrogate loss one can choose. A natural question is
thus ``which is best?''.
In order to do this we need to first define how we are 
evaluating losses as surrogates. 
To do this we require notation to describe the set of minimisers of the 
conditional and full risk associated with a loss.
Given a loss $\ell \colon \{-1,1\} \times \Vcal \to \RR$ its 
\emph{conditional minimisers at $\eta\in[0,1]$} is the set
\begin{equation}
	H(\ell,\eta) 
	:= \{ v\in\Vcal \colon L(\eta,v) = \minL(\eta) \}.
	\label{eq:H-def}
    \end{equation}
Given a set of hypotheses $\Hcal\subseteq\Vcal^\Xcal$, the (constrained)
Bayes optimal risk is 
\[
\minLL_\Hcal:= \inf_{h\in\Hcal} \LL(h,\PP).
\]
The \emph{(full) minimisers over $\Hcal$ for $\PP$} is the set
\[
	\Hcal(\ell,\PP)
	:= \{ h\in\Hcal \colon \LL(h) = \minLL_\Hcal \},
\]
where $\Hcal \subseteq \Vcal^{\Xcal}$ is some restricted set of functions
and $\LL(h) := \EE_{(\Xsf,\Ysf)\sim\PP}[\ell(\Ysf,h(\Xsf))]$ and the expectation 
is with respect to~$\PP$.
Given a \emph{reference loss} $\ell_{\mathrm{ref}}$, we will say the 
\emph{$\ell_{\mathrm{ref}}$-surrogate penalty} of a loss $\ell$ over the function class 
$\Hcal$ on a problem $(\eta,M)$ (or equivalently $\PP$)  is 
\[
S_{\ell_{\mathrm{ref}}}(\ell,\eta,M)=
S_{\ell_{\mathrm{ref}}}(\ell,\PP)
:= \inf_{h\in\Hcal(\ell,\PP)} \LL_{\mathrm{ref}}(h),
\]
where it is important to remember that $\LL$ is with respect to $\PP$.
That is, $S_{\ell_{\mathrm{ref}}}(\ell,\PP)$ is the minimum $\ell_{\mathrm{ref}}$ risk obtainable by a
function in $\Hcal$ that minimises the $\ell$ risk.

Given a fixed experiment $\PP$, if $\Lcal$ is a class of losses then the 
\emph{best surrogate losses in $\Lcal$} for the reference loss
$\ell_{\mathrm{ref}}$ are 
those that minimise the $\ell_{\mathrm{ref}}$-surrogate penalty.
This definition is motivated by the manner in which surrogate losses are used
--- one minimizes $\LL(h)$ over $h$ to obtain the minimiser $h^*$ and one hopes
that $\LL_{\mathrm{ref}}(h^*)$ is small.
Clearly, if the class of losses contains the reference loss 
(\emph{i.e.}, $\ell_{\mathrm{ref}} \in \Lcal$) then $\ell_{\mathrm{ref}}$ will be a best surrogate loss.
Therefore, the question of best surrogate loss is only interesting when 
$\ell_{\mathrm{ref}} \notin \Lcal$. 
One particular case we will consider is when the reference loss is the 0-1 loss 
and the class of surrogates $\Lcal$ is the set of convex proper losses. 
Since 0-1 loss is not convex the question of which surrogate is best is 
non-trivial. 

It would be nice if one could reason about the ``best'' surrogate loss using
the conditional perspective (that is working with $L$ instead of $\LL$) and in
a manner independent of $\Hcal$. It is simple to see why this can not be done. 
Since all the losses we consider are proper, the minimiser over $\heta$
of $L(\eta,\heta)$ is $\eta$. Thus any proper loss would lead to the same
$\heta\in[0,1]$. It is only the introduction of the restricted class of
hypotheses $\Hcal$ that prevents this reasoning being applied for $\LL$:
restrictions on $h\in\Hcal$ prevent $h(x)=\eta(x)$ for all $x\in\Xcal$. We
conclude that the problem of best surrogate loss only makes sense when one both
takes expectations over $\mathsf{X}$ \emph{and} restricts the class of
hypotheses $h$ to be drawn from some set $\Hcal\subsetneq[0,1]^\Xcal$.

This reasoning accords with that of \citet{Nock:2008,Nock2009} who examined
which surrogate to use and proposed a data-dependent scheme that tunes
surrogates for a problem.  They explicitly considered proper losses  and said
that ``minimizing any [lower-bounded, symmetric proper] loss amounts to the
\emph{same} ultimate goal'' and concluded that ``the crux of the choice of the
[loss] relies on data-dependent considerations''.

We demonstrate the difficulty of finding a universal best surrogate 
loss in by constructing a  simple example.
One can construct experiments $(\eta_1,M)$ and
$(\eta_2,M)$ and proper losses $\ell_1$ and $\ell_2$ such that   
\[
S_{\ell_{0-1}}(\ell_1,(\eta_1,M)) > S_{\ell_{0-1}}(\ell_2,(\eta_1,M)) 
\mbox{\ \  but\ \ }
S_{\ell_{0-1}}(\ell_1,(\eta_2,M)) < S_{\ell_{0-1}}(\ell_2,(\eta_2,M)).
\]
(The examples we construct have weight functions that
``cross-over'' each other; the details are in 
Appendix~\ref{section:appendix}.)
However, this does not
imply there can not exist a particular convex $\ell^*$ that minorises all proper
losses in this sense. 
Indeed, we conjecture that, in the sense described above, there is no best 
proper, convex surrogate loss.
\begin{conjecture}
    Given a proper, convex loss $\ell$ there exists a second proper, convex loss
    $\ell^* \ne \ell$, a hypothesis class $\Hcal$, and an experiment $\PP$ 
    such that
    \(
        S_{\ell_{0-1}}(\ell^*,\PP) < S_{\ell_{0-1}}(\ell,\PP)
    \)
    for the class $\Hcal$.
\end{conjecture}
To prove the above conjecture it would suffice to show that for a fixed 
hypothesis class and any pair of losses one can construct two experiments such 
that one loss minorises the other loss on one experiment and \emph{vice versa} 
on the other experiment.

Supposing the above conjecture is true, one might then ask for a best surrogate
loss for some reference loss $\ell_{\text{ref}}$ in a minimax sense. 
Formally, we would like the loss $\ell^*\in\Lcal$ such that the worst-case 
penalty for using $\ell^*$,
\[
    \Upsilon_{\Lcal}(\ell^*) 
	:= \sup_{\PP} \left\{
		S_{\ell_{\text{ref}}}(\ell^*,\PP) 
		- \inf_{\ell\in\Lcal} S_{\ell_{\text{ref}}}(\ell,\PP)
		\right\}
\]
is minimised. 
That is, $\Upsilon_{\Lcal}(\ell^*) \le \Upsilon_{\Lcal}(\ell)$ for all 
$\ell\in\Lcal$.

\subsection{The ``Minimal'' Symmetric Convex Proper Loss}
\label{subsection:minimal-loss}

Theorem \ref{theorem:simple-characterisation-of-convexity}
suggests an answer to the question
``What is the proper convex loss closest to the 0-1 loss?'' A way of making
this question precise  follows. Since $\ell$ is presumed proper, it has a
weight function $w$. Suppose w.l.o.g.~that $w(\half)=1$. Suppose the link is the
identity. The constraints in (\ref{eq:general-characterisation}) imply that the
weight function that is most similar to that for 0-1 loss meets the
constraints. Thus from (\ref{eq:double-inequality-simple})
\begin{equation}
    w^{\mathrm{minimal}}(c)=\frac{1}{2}\left(\frac{1}{c} \wedge \frac{1}{1-c}\right)
\label{eq:minimal-convex-w}
\end{equation}
is the weight for the convex proper loss closest to 0-1 
loss in this sense.  It is the weight function that forms the lower envelope of
the shaded region in the left diagram of Figure~\ref{fig:convexity-region}.
Using (\ref{eq:partial-losses}) one  can readily compute the corresponding
partial losses explicitly
\begin{equation}
    \ell_{-1}^{\mathrm{minimal}}(\heta)= \frac{1}{2}\left( \test{\heta<\textstyle\frac{1}{2}}
(-\heta-\ln(1-\heta)) + \test{\heta\ge
\textstyle\frac{1}{2}}(\heta-1-\ln(\textstyle\frac{1}{2}))\right)
\label{eq:ell-minus-1-minimal}
\end{equation}
and
\begin{equation}
    \ell_{1}^{\mathrm{minimal}}(\heta)=\frac{1}{2}\left(\test{\heta<\textstyle\frac{1}{2}}(-\heta-\log(\textstyle\frac{1}{2}))+\test{\heta\ge\textstyle\frac{1}{2}}(\heta-1-\ln\heta)\right).
\label{eq:ell-plus-1-minimal}
\end{equation}
Observe that the partial losses are (in part) linear, which is unsurprising as
linear functions are on the boundary of the set convex functions. This loss is
also best in another more precise (but ultimately unsatisfactory) sense, as we
shall now show.

Surrogate regret bounds are theoretical bounds on the regret of a desired loss
(say 0-1 loss) in terms of the regret with respect to a surrogate. 
\citet{ReidWilliamson2009} have shown the following (we only quote the simpler
symmetric case here):
\begin{theorem}
    Suppose $\ell$ is a proper loss with corresponding conditional Bayes risk
    $\minL$ which is symmetric about $\frac{1}{2}$: $\minL(\frac{1}{2}-c)
    =\minL(\frac{1}{2}+c)$ for $c\in[0,\frac{1}{2}]$. If the regret for the
    $\ell_{\frac{1}{2}}$ loss $\Delta L_{\frac{1}{2}}(\eta,\heta)=\alpha$, then
    the regret $\Delta L$ with respect to $\ell$ satisfies
    \begin{equation}
	\Delta L(\eta,\heta) \ge \minL(\textstyle\frac{1}{2})
	-\minL(\frac{1}{2}+\alpha).
	\label{eq:symmetric-regret-bound}
    \end{equation}
\end{theorem}
The bound in the theorem can be inverted to upper bound $\Delta L_{\frac{1}{2}}$
given an upper bound on $\Delta L(\eta,\heta)$. Considering all symmetric proper
losses normalised such that $w(\frac{1}{2})=1$, the right side of
(\ref{eq:symmetric-regret-bound}) is maximised and thus the bound on
$\Delta L_{\frac{1}{2}}$ in terms of $\Delta L$ is minimised when $\minL(\frac{1}{2}+\alpha)$
is maximised (over all losses normalised as mentioned). But since $w=-\minL''$,
that occurs for the pointwise minimiser of $w$ (subject to
$w(\half)=1$). Since we are interested in convex losses, the minimising $w$ is
given by (\ref{eq:minimal-convex-w}). In this case the right hand side of
(\ref{eq:symmetric-regret-bound}) can be
explicitly determined to be $(\frac{\alpha}{2}+\frac{1}{4})\log(2\alpha+1)-\frac{\alpha}{2}$, and the bound can be inverted to obtain the result that if
$\Delta L^{\mathrm{minimal}}(\eta,\heta)=x$ then 
\begin{equation}
    \Delta L_{\frac{1}{2}}(\eta,\heta)\le
    \frac{1}{2}\exp\left(\mathrm{LambertW}\left(
    \frac{(4x-1)}{\mathrm{e}}\right)+1\right)-\frac{1}{2}
    \label{eq:minimal-regret-bound}
\end{equation}
which is plotted in Figure \ref{figure:regret-bound}.
\begin{figure}
    \begin{center}
    \includegraphics[width=0.5\textwidth]{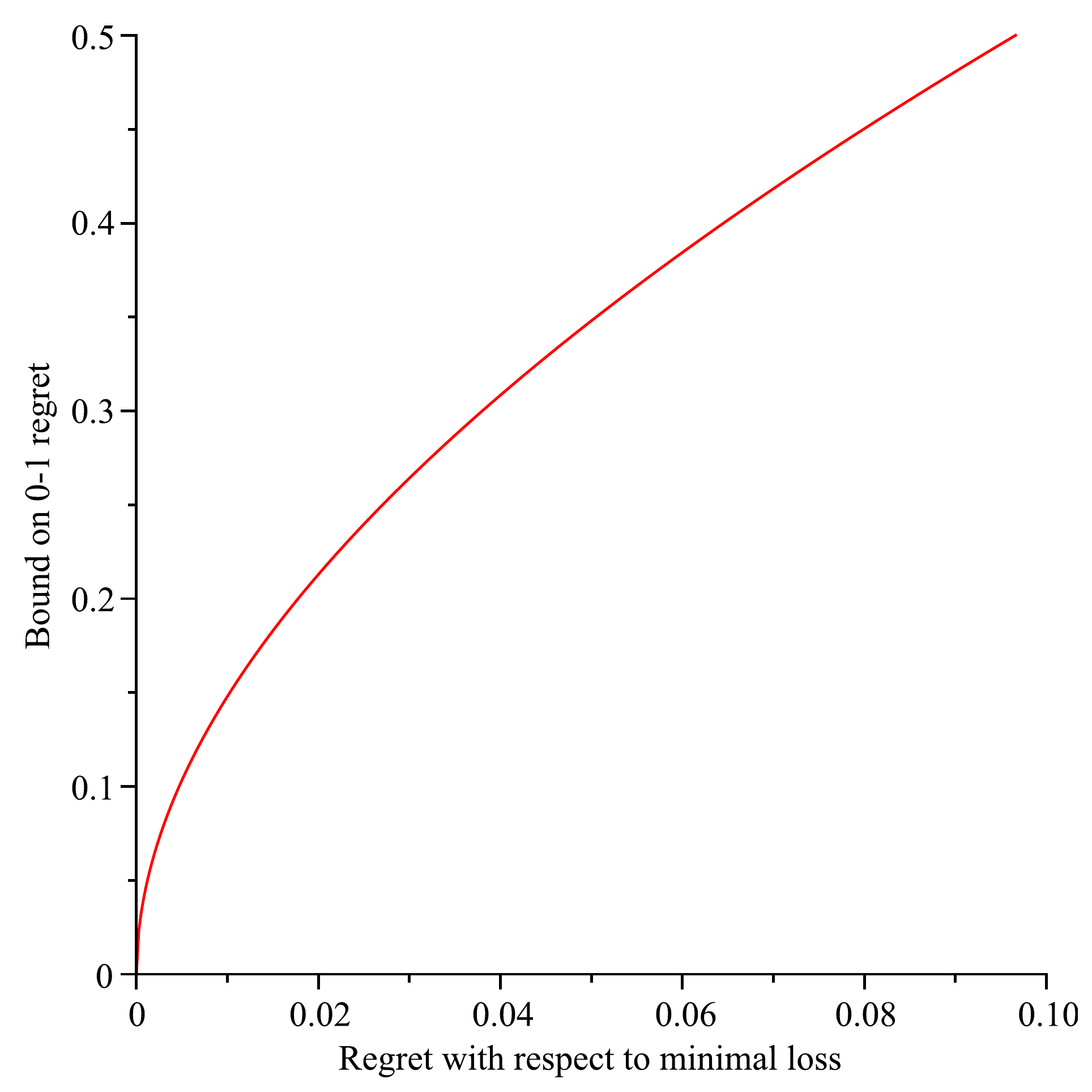}
\end{center}
\caption{Upper bound on the 0-1 regret in terms of
$\Delta L^{\mathrm{minimal}}$ as given by~(\ref{eq:minimal-regret-bound}).
\label{figure:regret-bound}}
\end{figure}

The above argument does \emph{not} show that the loss given by
(\ref{eq:ell-minus-1-minimal},\ref{eq:ell-plus-1-minimal}) is the \emph{best}
surrogate loss.
Nevertheless it does suggest it is at least
worth considering using $\ell^{\mathrm{minimal}}$ as a convex proper 
surrogate binary loss.

\section{Conclusions}
\label{section:conclusion}


Composite losses are widely used. In this paper we have characterised a number
of aspects of them: their relationship to margin losses, the connection between
properness and classification calibration, the constraints symmetry imposes,
when composite losses are convex, and natural ways to parametrise them. We have
also considered the question of the ``best'' surrogate loss.

The parametrisation of a composite loss in terms $(w,\psi')$ (or $\rho$)  has
advantages over using $(\phi,\psi)$ or $(\minL,\psi)$. As explained by
\citet{Masnadi-Shirazi:2009}, the representation in terms of $(\phi,\psi)$ is
in general not unique. The representation in terms of $\minL$ is harder to
intuit: whilst indeed the Bayes risk for squared loss and 0-1 loss are
``close'' (compare the graph of $c\mapsto c(1-c)$ with that of $c\mapsto
c\wedge (1-c)$), by examining their weight functions they are seen to be very
different ($w(c)=1$ versus $w(c)=2\delta(c-\frac{1}{2})$). We have also seen
that on the basis of Theorem~\ref{theorem:general-characterisation}, the
parametrisation $(w,\psi')$ is perhaps the most natural --- there is a pleasing
symmetry between the loss and the link as they are in this form both
parametrised in terms of non-negative weight functions on $[0,1]$. Recall too
that the canonical link sets $\psi'$ equal to $w$.

The observation suggests an alternate inductive principle known as 
\emph{surrogate tuning}, which seems to have been first suggested by
\citet{Nock:2008}.  The idea of surrogate tuning is simple: noting that the
best surrogate depends on the problem, adapt the surrogate you are using to the
problem. In order to do so it is important to have a good parametrisation of
the loss. The weight function perspective does just that, especially given
Theorem \ref{theorem:simple-characterisation-of-convexity}. It would be
straight forward to develop low dimensional parametrisations of $w$ that
satisfy the conditions of this theorem which would thus allow a learning
algorithm to explore the space of convex losses. One could (taking due care
with the subsequent multiple hypothesis testing problem) regularly
\emph{evaluate} the 0-1 loss of the hypotheses so obtained. The observations
made in section~\ref{sec:composite-losses} regarding stochastic gradient
descent algorithms may be of help in this regard.

Surrogate tuning differs from loss \emph{tailoring}
\citep{Hand1994,HandVinciotti2003,Buja:2005} which involves adapting the loss
to what you really think is important. In the surrogate tuning setting, we have
fixed on 0-1 loss as what we really want to minimise and use a surrogate solely
for computational reasons.

Finally, we conjecture that 
$\ell^{\mathrm{minimal}}$ (equations \ref{eq:ell-minus-1-minimal}
and \ref{eq:ell-plus-1-minimal}) is somehow
special in the class of proper convex losses in some  way other than being the
pointwise minimiser of weights (and the normalised loss with smallest regret
bound with respect to $\ell^{0-1}$), 
but the exact nature of the
specialness still eludes us. Perhaps it is optimal in some weaker (minimax)
sense. The reason for this suggestion is that it is not hard to show that for
reasonable $\PP$ there exists $\Hcal$ such that $c\mapsto \LL_c(h,\PP)$ takes
on all possible values within the constraints
\[
0 \le \LL_c(h,\PP)\le \max(c, 1-c) 
\]
which follows immediately from the definition of cost-sensitive
misclassification loss. Furthermore the example in the appendix below seems to
require loss functions whose corresponding weight functions cross over each
other and there is no weight function corresponding to a convex proper loss
that crosses over $w^{\mathrm{minimal}}$.

\paragraph*{Acknowledgements} 
This work was motivated in part by a question due to John Langford.
Thanks to Fangfang Lu for discussions and finding several bugs in an earlier
version. Thanks to Ingo Steinwart for pointing out the $\eta_\alpha$ trick.
Thanks to Tim van Erven for comments and corrections.
This work was supported by the Australian Research Council and NICTA through
Backing Australia's Ability.

\appendix
\section{Example Showing Incommensurability of Two Proper Surrogate Losses}
\label{section:appendix}

We consider $\Xcal=[0,1]$ with $M$ being uniform on $\Xcal$, 
and consider the two problems that are induced by 
\[
\eta_1(x)=x^2\ \ \ \mbox{and} \ \ \ 
\eta_2(x)=\frac{1}{3}+\frac{x}{3}.
\]
We use a simple linear hypothesis class
\[
    \Hcal:=\{h_\alpha(x):=\alpha x \colon \alpha\in[0,1]\},
\]
with identity link function
and consider the two  surrogate proper losses $\ell_1$ and $\ell_2$
with weight functions
\[
w_1(c)=\frac{1}{c},\ \ \ \ \ \ \ \ \ \  \ w_2(c)=\frac{1}{1-c}.
\]
These weight functions correspond to the two curves that construct the left
diagram in Figure
\ref{fig:convexity-region}.
The corresponding conditional losses can be readily calculated to be
\begin{eqnarray*}
    L_1(\eta,h) &:=& \eta(h-1-\log(h)) + (1-\eta)h\\
    L_2(\eta,h) &:=& \eta(1-h) +(1-\eta)(-h-\log(1-h)).
\end{eqnarray*}
One can numerically compute the parameters for the constrained Bayes optimal
for each problem and for each surrogate loss:
\begin{eqnarray*}
\alpha_{1,1}^* &=& \argmin_{\alpha\in[0,1]} \LL_1(\eta_1,h_\alpha,M) = 0.66666667\\
\alpha_{2,1}^* &=& \argmin_{\alpha\in[0,1]} \LL_2(\eta_1,h_\alpha,M) = 0.81779259\\
\alpha_{1,2}^* &=& \argmin_{\alpha\in[0,1]} \LL_1(\eta_2,h_\alpha,M) = 1.00000000\\
\alpha_{2,1}^* &=& \argmin_{\alpha\in[0,1]} \LL_2(\eta_2,h_\alpha,M) =
0.77763472.
\end{eqnarray*}
Furthermore
\begin{eqnarray*}
    & & \LL_{0-1}(\eta_1,h_{\alpha_{1,1}^*},M)= 0.3580272,\ \ \ \ \ \ \ 
    \LL_{0-1}(\eta_1,h_{\alpha_{2,1}^*},M)= 0.3033476,\\
    & & \LL_{0-1}(\eta_2,h_{\alpha_{1,2}^*},M)= 0.4166666,\ \ \ \ \ \ \ 
    \LL_{0-1}(\eta_2,h_{\alpha_{2,2}^*},M)= 0.4207872.
\end{eqnarray*}
Thus for problem $\eta_1$ the surrogate loss $L_2$ has a constrained Bayes
optimal hypothesis $h_{\alpha_{2,1}^*}$ which has a lower 0-1 risk than the
constrained Bayes optimal hypothesis $h_{\alpha_{1,1}^*}$ for the surrogate
loss $L_1$. Thus for problem $\eta_1$ surrogate $L_2$ is better than surrogate
$L_1$. However for problem $\eta_2$ the situation is reversed: surrogate $L_2$
is \emph{worse} than surrogate $L_1$.

\section{An Alternate View of Canonical Links}
\label{section:appendix-canonical}

This appendix contains an alternate approach to understanding canonical links
using convex duality. In doing so we present an improved formulation of a
result on the duality of Bregman
divergences that may be of independent interest.

The \emph{Legendre-Fenchel} (LF) dual
 $\lfdual{\phi}$ of a function $\phi \colon \RR \to \RR$ is a function
defined by
\begin{equation}\label{eq:lfdual}
	\lfdual{\phi}(\lfdual{s})
	:= \sup_{s\in\RR} \{ \inner{s}{\lfdual{s}} - \phi(s) \}.
\end{equation}
The LF dual of any function is convex.

When $\phi(s)$ is a function of a real argument $s$ and the derivative 
$\phi'(s)$ exists, the Legendre-Fenchel conjugate $\lfdual{\phi}$ is given by the 
\emph{Legendre transform} \citep{Rockafellar:1970, hiriarturruty2001fca}
\begin{equation}
    \lfdual{\phi}(s)
    = s\cdot (\phi')^{-1}(s) - \phi\left( (\phi')^{-1}(s)\right).
    \label{eq:legendre-transform}
\end{equation}
Thus (writing $\partial f:=f'$) $f'=(\partial\lfdual{f})^{-1}$. Thus with
$w$, $W$, and $\Wb$ defined as above,
\begin{equation}
    W=(\partial(\lfdual{\Wb}))^{-1},\ \ \ W^{-1}=\partial(\lfdual{\Wb}), \ \
    \lfdual{\Wb}=\int W^{-1} .
    \label{eq:W-star-relationships}
\end{equation}

Let $w$, $W$, $\Wb$ be as in Theorem~\ref{theorem:loss-from-weight}.
Denote by $L_W$ the
$w$-weighted conditional loss parametrised by $W=\int w$ and let
$\Delta L_W$ be the corresponding regret (we can interchange $\Delta L$ and $D$
here by 
(\ref{eq:bregman-general-loss}) since $\psi_L=\mathrm{id}$.
    \begin{equation}
	D_w(\eta,\heta)= \Wb(\eta) -\Wb(\heta) -(\eta-\heta) W(\heta) .
	\label{eq:B-w-general-form}
    \end{equation}
We now further consider $D_w$ as given by (\ref{eq:B-w-general-form}). It will
be convenient to parametrise $D$ by $W$ instead of $w$. Note that the standard
parametrisation for a Bregman divergence is in terms of the convex function
$\Wb$. Thus will write $D_{\Wb}$, $D_W$ and $D_w$  to all represent
(\ref{eq:B-w-general-form}).
The following theorem is known (e.g \cite{Zhang:2004}) 
but as will be seen, stating it
in terms of $D_W$ provides some advantages.
\begin{theorem}
    Let $w$, $W$, $\Wb$ and $D_W$ be as above. Then for all
    $x,y\in[0,1]$,
    \begin{equation}
	D_W(x,y) = D_{W^{-1}}(W(y),W(x)).
	\label{eq:B-W-dual}
    \end{equation}
\end{theorem}
\begin{proof}
    Using (\ref{eq:legendre-transform}) we have
    \begin{eqnarray}
    & & 	\lfdual{\Wb}(u) = u\cdot W^{-1}(u) -\Wb(W^{-1}(u)) \nonumber\\
    &\Rightarrow& \Wb(W^{-1}(u)) = u\cdot W^{-1}(u) - \lfdual{\Wb}(u).
    \label{eq:W-bar-W-inverse}
    \end{eqnarray}
    Equivalently (using (\ref{eq:W-star-relationships}))
    \begin{equation}
	\lfdual{\Wb}(W(u)) = u\cdot W(u) -\Wb(u).
	\label{eq:W-bar-star-W}
    \end{equation}
    Thus substituting and then using (\ref{eq:W-bar-W-inverse}) we have
    \begin{eqnarray}
	D_W(x,W^{-1}(v)) &=& \Wb(x) - \Wb(W^{-1}(v)) - (x-W^{-1}(v))\cdot
	    W(W^{-1}(v))\nonumber\\
	&=& \Wb(x)+\lfdual{\Wb}(v) -v W^{-1}(v) - (x-W^{-1}(v))\cdot v
	\nonumber\\
	&=& \Wb(x) +\lfdual{\Wb}(v) - x\cdot v. \label{eq:B-W-int-res-1}
    \end{eqnarray}
    Similarly (this time using (\ref{eq:W-bar-star-W}) we have
    \begin{eqnarray}
	 D_{W^{-1}}(v,W(x))&=&\lfdual{\Wb}(v)-\lfdual{\Wb}(W(x))
	 -(v-W(x))\cdot W^{-1}(W(x))\nonumber\\
	 &=&\lfdual{\Wb}(v)-x W(x)+\Wb(x)-v\cdot x +x W(x)\nonumber\\
	 &=& \lfdual{\Wb}(v)+\Wb(x) - v\cdot x\label{eq:B-W-int-res-2}
    \end{eqnarray}
    Comparing (\ref{eq:B-W-int-res-1}) and (\ref{eq:B-W-int-res-2}) we  see
    that
    \[
       D_W(x,W^{-1}(v)) = D_{W^{-1}}(v,W(x)) 
    \]
    Let $y=W^{-1}(v)$. Thus subsitituting $v=W(y)$ leads to
    (\ref{eq:B-W-dual}).
\end{proof}
The weight function corresponding to $D_{W^{-1}}$ is 
$
\frac{\partial}{\partial x} W^{-1}(x) = \frac{1}{w(W^{-1}(x))} .
$

\begin{theorem}
    If the inverse link $\psi^{-1}=W^{-1}$ (and thus $\heta=W^{-1}(\hh)$) then
    \begin{eqnarray*}
	D_W(\eta,\heta) &=& D_W(\eta,W^{-1}(\hh)) =
	\Wb(\eta)+\lfdual{\Wb}(\hh) -\eta\cdot\hh\\
	L_W(\eta,\heta) &=& L_W(\eta,W^{-1}(\hh)) = \lfdual{\Wb}(\hh) -
	\eta\cdot\hh + \eta(\Wb(1)+\Wb(0))-\Wb(0)\\
	\frac{\partial}{\partial\hh} L_W(\eta,W^{-1}(\hh)) & =& \heta-\eta
    \end{eqnarray*}
    and furthermore $D_W(\eta,W^{-1}(\hh))$ and $L_W(\eta,W^{-1}(\hh))$ are
    convex in $\hh$.
\end{theorem}
\begin{proof}
    The first two expressions follow immediately from (\ref{eq:B-W-int-res-1})
    and (\ref{eq:B-W-int-res-2}) by substitution. The derivative follows from
    calculation: $\frac{\partial}{\partial\hh} L_W(\eta,W^{-1}(\hh)) 
    = \frac{\partial}{\partial\hh}(\lfdual{\Wb}(\hh)-\eta\cdot\hh) 
    = W^{-1}(\hh)-\eta=\heta-\eta$.  The convexity follows from the fact that 
    $\lfdual{\Wb}$ is
    convex (since it is the LF dual of a convex function $\Wb$) and the overall
    expression is the sum of this and a linear term, and thus convex.
\end{proof}
\cite{Buja:2005} call $W$ the \emph{canonical link}.  We have already seen
(Theorem~\ref{theorem:composite-canonical}) that the composite loss 
constructed using the canonical link is convex.

\section{Convexity and Robustness}
\label{section:appendix-robustness}

In this appendix we show how the characterisation of the convexity of proper
losses (Theorem \ref{theorem:simple-characterisation-of-convexity}) allows one
to make  general algorithm independent statements about the robustness of
convex proper losses to random mis-classification noise.

\citet{Long:2008} have shown that  that boosting with convex
potential functions (i.e., convex margin losses) is not robust to random class
noise\footnote{
    We define exactly what we mean by robustness below. The notion that
    \citet{Long:2008} examine is akin to that studied for instance by 
    \citet{Kearns1998}. There are many other meanings of ``robust'' which are 
    different to that which we consider. The classical notion of robust 
    statistics \citep{Huber1981} is motivated by robustness to 
    contamination of additive observation noise 
    (some heavy-tail noise mixed in with the Gaussian noise often assumed 
    in designing estimators). There are some results about particular machine 
    learning algorithms being robust in that sense \citep{SchSmoWilBar00}.
    ``Robust''  is also used to mean robustness with respect to random 
    attribute noise \citep{Trafalis2006}, robustness to unknown prior class
    probabilities \citep{ProvostFawcett2001}, or a Huber-style robustness to
    attribute noise (``outliers'') for classification \citep{Fidler2006}. We
    only study robustness in the sense of random label noise.
}. That is, they are 
susceptible to random class noise. 
In particular they present a very simple learning task which is
``boostable'' -- can be perfectly solved using a linear combination of base
classifiers -- but for which, in the presence of any amount of label noise,
idealised, early stopping and $L_1$ regularised boosting algorithms will learn
a classifier with only 50\% accuracy.

This has led to the recent proposal of 
boosting algorithms that use non-convex margin losses and experimental evidence
suggests that these are more robust to class noise than their convex 
counterparts.
\citet{Freund:2009} recently described RobustBoost,
which uses a parameterised family of non-convex surrogate losses that
approximates the 0-1 loss as the number of boosting iterations increases.
Experiments on a variant of the task proposed by
\citet{Long:2008} show that RobustBoost is very insensitive to class noise.
\citet{Masnadi-Shirazi:2009} presented
SavageBoost, a boosting algorithm built upon a non-convex margin function.
They argued that even when the margin function is non-convex the conditional
risk may still be convex. We elucidate this via  our characterisation of the 
convexity of composite losses. 
Although all these results are suggestive, it is not clear from
these results whether the robustness or not is a property of the loss function,
the algorithm or a combination. We study that question by
considering robustness in an algorithm-independent fashion.

For $\alpha\in(0,\half)$ and $\eta\in[0,1]$ we will define
\[
    \eta_\alpha := \alpha(1-\eta) + (1-\alpha)\eta
\]
as the 
\emph{$\alpha$-corrupted version} of $\eta$.
This captures the idea that instead of drawing a positive label for the point 
$x$ with probability $\eta(x)$  there is a random class flip with
probability $\alpha$.
Since $\eta_\alpha$ is a convex combination of $\alpha$ and $1 - \alpha$ it 
follows that $\eta_\alpha \in [\alpha,1-\alpha]$.
The effect of $\alpha$-corruption on the conditional risk of a loss can be seen 
as a transformation of the loss.

\begin{lemma}\label{lem:noise-loss}
	If $\ell^\psi$ is any composite loss then its conditional risk satisfies
	\[
	L^\psi(\eta_\alpha,v) = L^\psi_\alpha(\eta,v), \ \ \ \eta\in [0,1], \ \
	v\in\mathcal{V},
	\]
	where 
	$\ell^\psi_\alpha(y,v) = (1-\alpha)\ell^\psi(y,v) + \alpha\ell^\psi(-y,v)$.
\end{lemma}
\begin{proof}
	By simple algebraic manipulation we have
	\begin{eqnarray*}
		L^\psi(\eta_\alpha,v)
		& = & (1-\eta_\alpha)\ell^\psi(-1, v) + \eta_\alpha\ell^\psi(1,v) \\
		& = & [(1-\alpha)(1-\eta)+\alpha\eta]\ell^\psi(-1,v)
				+ [\alpha(1-\eta)+(1-\alpha)\eta]\ell^\psi(1,v) \\
		& = & (1-\eta)[(1-\alpha)\ell^\psi(-1,v) + \alpha\ell^\psi(1,v)]
				+ \eta[\alpha\ell^\psi(-1,v) + (1-\alpha)\ell^\psi(1,v)] \\
		& = & (1-\eta)\ell^\psi_\alpha(-1,v) + \eta\ell^\psi_\alpha(1,v) \\
		& = & L_\alpha^{\psi}(\eta,v)
	\end{eqnarray*}
	proving the result.
\end{proof}

In particular, if $\ell$ is strictly proper then $\ell_\alpha$ cannot be 
proper because the minimiser of $L(\eta_\alpha,\cdot)$ is $\eta_\alpha$ and so 
$\eta_\alpha \ne \eta$ must also be the minimiser of $L_\alpha(\eta,\cdot)$.
This suggests that strictly proper losses are not robust to any class noise.

\subsection{Robustness implies Non-convexity}
We now define a general notion of robustness for losses for class probability 
estimation.
\begin{definition}
Given an $\alpha\in[0,\half)$, we will say a loss 
$\ell \colon \{-1,1\} \times [0,1] \to \RR$ is 
\emph{$\alpha$-robust at $\eta$} if the set of minimisers of the conditional 
risk for $\eta$ and the set of minimisers of the conditional risk for 
$\eta_\alpha$ have some common points.
\end{definition}
That is, a loss is $\alpha$-robust for a particular $\eta$ if minimising the 
noisy conditional risk can potentially give an estimate that is also a 
minimiser of the non-noisy conditional risk.
Formally, \emph{$\ell$ is $\alpha$-robust at $\eta$} when
\[
	H(\ell, \eta_\alpha) \cap H(\ell, \eta) \ne \emptyset,
\]
where $H(\ell,\eta)$ is defined in (\ref{eq:H-def}).

Label noise is symmetric about $\half$ and so the map $\eta \mapsto \eta_\alpha$
preserves the side of $\half$ on which the values $\eta$ and $\eta_\alpha$ are 
found.
That is, $\eta \le \half$ if and only if $\eta_\alpha \le \half$ for all 
$\alpha \in [0,\half)$.
This means that 0-1 misclassification loss or, equivalently,
$\ell_{\half}$ is
$\alpha$-robust for all $\eta$ and for all $\alpha$.
For other $c$, the range of $\eta$ for which $\ell_c$ is $\alpha$-robust is
more limited.



\begin{theorem}
	For each $c\in(0,1)$, the loss $\ell_c$ is $\alpha$-robust at $\eta$
	if and only if
	\[
	\eta \notin \left[\frac{c-\alpha}{1-2\alpha}, c\right)\ \ \mbox{for}\
	c < \thalf\ \ \ \mbox{or}\ \ \ 
	\eta \notin \left[c, \frac{c-\alpha}{1-2\alpha}\right)\ \mbox{for}\ 
	c \ge \thalf.
	\]
	\label{theorem:alpha-robustness}
\end{theorem}
\begin{proof}
	By the definition of $L_c$ and $\test{\heta < c} = 1 - \test{\heta \ge c}$
	we have
	\[
		L_c(\eta,\heta)
		= (1-\eta)c\test{\heta \ge c} + \eta(1-c)\test{\heta < c} 
		= \eta(1-c) + (c-\eta)\test{\heta \ge c}.
	\]
	Since $c-\eta$ is positive iff $c>\eta$ we see $L_c(\eta,\heta)$ is
	minimised for $\eta<c$ when $\heta < c$ and for $\eta \ge c$ when 
	$\heta \ge c$. 
	So $H(\ell_c,\eta) = [0,c)$ for $\eta<c$ and $H(\ell_c,\eta) = [c,1]$ for 
	$\eta\ge c$.
	Since $[0,c)$ and $[c,1]$ are disjoint for all $c\in[0,1]$ we see that
	$H(\ell_c,\eta)$ and $H(\ell_c,\eta_\alpha)$ coincide if and only if 
	$\eta,\eta_\alpha < c$ or $\eta,\eta_\alpha \ge c$ and are disjoint
	otherwise.


	We proceed by cases. 
	First, suppose $c<\half$.
	For $\eta<c<\half$ it is easy to show $\eta_\alpha \ge c$ iff 
	$\eta \ge \frac{c-\alpha}{1-2\alpha}$ and so $\ell_c$ is not 
	$\alpha$-robust for $\eta\in[\frac{c-\alpha}{1-2\alpha},c)$.
	For $c \le \eta$ we see $\ell_c$ must be $\alpha$-robust 
	since $\eta_\alpha < c$ iff $\eta < \frac{c-\alpha}{1-2\alpha}$
	but $\frac{c-\alpha}{1-2\alpha} < c$ for $c<\half$ which is a contradiction.
	Thus, for $c<\half$ we have $\ell_c$ is $\alpha$-robust 
	iff $\eta\notin[\frac{c-\alpha}{1-2\alpha},c)$.
	
	For $c>\half$ the main differences are that  
	$\frac{c-\alpha}{1-2\alpha} > c$ for $c>\half$ 
	and $\eta_\alpha < \eta$ for $\eta > \half$.
	Thus, by a similar argument as above we see that $\ell_c$ is 
	$\alpha$-robust iff $\eta\notin[c,\frac{c-\alpha}{1-2\alpha})$.
\end{proof}


This theorem allows us to characterise the robustness of arbitrary proper 
losses by appealing to the integral representation in 
(\ref{eq:schervish-representation}).

\begin{lemma}\label{lem:H-for-intrep}
	If $\ell$ is a proper loss with weight function $w$
	then $H(\ell,\eta) = \bigcap_{c\colon w(c)>0}H(\ell_c,\eta)$ and so
	\[
		H(\ell,\eta) \cap H(\ell,\eta_\alpha)
		= \bigcap_{c \colon w(c)>0} H(\ell_c,\eta) 
		\cap H(\ell_c,\eta_\alpha).
	\]	
\end{lemma}
\begin{proof}
	We first show that 
	$H(\ell,\eta) \subseteq \bigcap_{c\colon w(c)>0}H(\ell_c,\eta)$
	by contradiction.
	Assume there is an $\heta\in H(\ell,\eta)$ but for which there
	is some $c_0$ such that $w(c_0) > 0$ and 
	$\heta \notin H(\ell_{c_0},\eta)$.
	Then there is a $\heta'\in H(\ell_{c_0},\eta)$ and 
	$\heta'\in H(\ell_c)$ for all other $c$ for which $w(c) > 0$ 
	(otherwise $H(\ell,\eta) = \{\heta\}$).
	Thus, $L_{c_0}(\eta,\heta') < L_{c_0}(\eta,\heta')$ and so 
	$\int_0^1 L_c(\eta,\heta')\,w(c)\,dc < \int_0^1 L_c(\eta,\heta)\,w(c)\,dc$
	since $w(c_0) >0$.

	Now suppose $\heta \in \bigcap_{c\colon w(c)>0}H(\ell_c,\eta)$. 
	That is, $\heta$ is a minimiser of $L_c(\eta,\cdot)$ for all $c$ such that
	$w(c) > 0$ and therefore must also be a minimiser of
	$L(\eta,\cdot) = \int_0^1 L_c(\eta,\cdot)\,w(c)\,dc$ and is therefore
	in $H(\ell,\eta)$, proving the converse.	
\end{proof}
One consequence of this lemma is that if $w(c)>0$ and $\ell_c$ is not 
$\alpha$-robust at $\eta$ then, by definition, 
$H(\ell_c,\eta)\cap H(\ell_c,\eta_\alpha) = \emptyset$ and so $\ell$ cannot
be $\alpha$-robust at $\eta$.
This means we have established the following theorem regarding the 
$\alpha$-robustness of an arbitrary proper loss in terms of its weight function.

\begin{theorem}\label{thm:robust-weight}
	If $\ell$ is a proper loss with weight function $w$ then it is not 
	$\alpha$-robust for any 
	\[
	\eta \in \bigcup_{c\colon w(c)>0} \left[\frac{c-\alpha}{1-2\alpha}, c\right) 
	\cup \left[c, \frac{c-\alpha}{1-2\alpha}\right).
	\]
\end{theorem}

By Corollary~\ref{cor:convex-proper} we see that convex proper losses are
strictly proper and thus have weight functions which are non-zero for all 
$c\in[0,1]$ and so by Theorem~\ref{thm:robust-weight} we have the following
corollary.

\begin{corollary}\label{cor:convex-implies-non-robust}
	If a proper loss is convex, then for all $\alpha\in(0,\half)$ it is not 
	$\alpha$-robust at any $\eta\in[0,1]$.
\end{corollary}

At a high level, this result -- ``convexity implies non-robustness'' --
appears to be logically equivalent to Long and Servedio's result that 
``robustness implies non-convexity''. 
However, there are a few discrepancies that mean they are not directly 
comparable.
The definitions of robustness differ. 
We focus on the point-wise minimisation of conditional risk as this is, 
ideally, what most risk minimisation approach try to achieve.
However, this means that robustness of ERM with regularisation or restricted 
function classes is not directly captured with our definition whereas
Long and Servedio analyse this latter case directly.
In our definition the focus is on probability estimation robustness while
the earlier work is focussed on classification accuracy.
Our work could be extended to this case by 
analysing $H(\ell,\eta) \cap H(\ell_{\half},\eta)$.

Additionally, their work restricts attention to the robustness of boosting 
algorithms that use convex potential functions whereas our analysis is not
tied to any specific algorithm.
By restricting their attention to a specific learning task and class of 
functions they are able to show a very strong result: that convex losses
for boosting lead to arbitrarily bad performance with arbitrarily little noise.
Also, our focus on proper losses excludes some convex losses
(such as the hinge loss) that is covered by Long and Servedio's results.

%

Finally, it is worth noting that there are non-convex loss functions that are 
strictly proper and so are not robust in the sense we use here.
That is, the converse of Corollary~\ref{cor:convex-implies-non-robust}
is not true.
For example, any loss with weight function that sits above 0 but outside
the shaded region in Figure~\ref{fig:convexity-region} will be non-convex 
and non-robust.
This suggests that the arguments made by \citet{Masnadi-Shirazi:2009,Freund:2009}
for the robustness of non-convex losses need further investigation.

\bibliographystyle{plain}
\bibliography{SurrogateBinaryLosses}
\end{document}